%% file: long.tex
\documentclass{article}

\usepackage[a4paper, margin=1in]{geometry} 

\usepackage{amsmath, amssymb, amsthm}
\usepackage{mathrsfs}

\usepackage{natbib}
\let\cite\citep

\usepackage{graphicx}
\usepackage{subcaption}

\usepackage{hyperref}
\usepackage{cleveref}

\usepackage{algorithm}
\usepackage{algpseudocode}

\usepackage{mdframed}
\usepackage{booktabs}

\input{notation}

\title{
Decision Theoretic Foundations for Conformal Prediction: \\
Optimal Uncertainty Quantification for Risk-Averse Agents}
\author{%
  Shayan Kiyani, George Pappas, Aaron Roth, Hamed Hassani \\
  University of Pennsylvania\\
  \texttt{\{shayank, pappasg, aaroth, hassani\}@seas.upenn.edu} \\
  }
\date{\today}

\begin{document}

\maketitle

\begin{abstract}



A fundamental question in data-driven decision making is how to quantify the uncertainty of predictions in ways that can usefully inform downstream action.  This interface between prediction uncertainty and decision-making is especially important in risk-sensitive domains, such as medicine. 
In this paper, we develop decision-theoretic foundations that connect uncertainty quantification using prediction sets with risk-averse decision-making. Specifically, we answer three fundamental questions:  (1) What is the correct notion of uncertainty quantification for risk-averse decision makers? We prove that prediction sets are optimal for decision makers who wish to optimize their value at risk. (2) What is the optimal policy that a risk averse decision maker should use to map prediction sets to actions? We show that a simple max-min decision policy is optimal for risk-averse decision makers. Finally, (3) How can we derive prediction sets that are optimal for such decision makers? We provide an exact characterization in the population regime and a distribution free finite-sample construction. Answering these questions naturally leads to an algorithm, \emph{Risk-Averse Calibration (RAC)}, which follows a provably optimal design for deriving action policies from predictions. RAC is designed to be both \emph{practical}—capable of leveraging the quality of predictions in a black-box manner to enhance downstream utility—and \emph{safe}—adhering to a user-defined risk threshold and optimizing the corresponding risk quantile of the user's downstream utility.
Finally, we experimentally demonstrate the significant advantages of RAC in applications such as medical diagnosis and recommendation systems. Specifically, we show that RAC achieves a substantially improved trade-off between safety and utility, offering higher utility compared to existing methods while maintaining the safety guarantee.

   
\end{abstract}

\section{Introduction}\label{intro}
 Predictions are frequently used to inform \emph{actions}. For example, in clinical medicine, patient data are used to predict diagnoses and outcomes when choosing treatments. In high-stakes cases—where an incorrect treatment decision could lead to serious complications or death—it is crucial not to rely solely on a model's predictions. Instead, decisions must account for the uncertainty in these predictions, opting for more conservative interventions when that uncertainty makes the potential outcomes (e.g., complications, side effects) highly variable. Connecting uncertain predictions to actionable, principled decisions is a significant challenge in  safety-critical domains, including medical diagnosis, finance, robotics, and control, and requires balancing safety with utility.  At one extreme, we could ensure safety by avoiding any action, essentially ignoring predictions and thus failing to be of practical use. At the other extreme, we could aggressively exploit predictions attempting to maximize expected utility, trading off high (expected) upside gains for significant downside risk --- at the cost of realizing poor outcomes with significant probability.  Balancing this trade-off requires optimally integrating prediction into decision making in a way that is risk-sensitive. To this end, we focus on the following question:
\begin{center}
\emph{What is the optimal interface between prediction and action that allows for navigating the trade-off between safety   and utility in high stakes applications?}
\end{center}


The optimal design of an action policy hinges on optimal uncertainty quantification.  Among the various methods for quantifying uncertainty, a widely adopted approach—driven by advances in conformal prediction—is to produce \emph{prediction sets} rather than point predictions. But what are prediction sets good for? In particular what kind of downstream decision-making process makes prediction sets the right language for communicating uncertainty? And given such a process, what is the optimal decision-making rule for transforming prediction sets into actions? 
To begin answering these questions, we first introduce our basic setting and notation. We are given a space of features (or covariates, contexts) $\mathcal{X}$ and a set of labels $\mathcal{Y}$, where the pair $(x, y) \in \mathcal{X} \times \mathcal{Y}$ is generated according to a distribution $\mathcal{D}$. A downstream decision maker has an action set $\mathcal{A}$ and a utility function $u: \mathcal{A} \times \mathcal{Y} \rightarrow \mathbb{R}$ that maps actions $a$ and realized labels $y$ to utilities $u(a, y)$, which the decision maker seeks to maximize. Upon observing $x \in \mathcal{X}$, the decision maker must take an action $a \in \mathcal{A}$ without observing the true label $y$, relying instead on predictions about $y$. Within this framework, we aim to answer the above questions.

In thinking about what kind of answers the above questions might have, it is useful to reflect on what we can say about \emph{calibrated forecasts}, an alternative way of quantifying uncertainty---that has its own limitations---but which has a rigorous decision-theoretic foundation.  Suppose that we are in a multi-class  classification setting, and that we represent labels using their one-hot encoding --- so that we can view both labels $y\in \cY$ and distributions over labels $p$ as vectors in the $k$-dimensional probability simplex. A forecasting rule $f:\cX\rightarrow [0,1]^k$ is said to be calibrated if for every prediction $\hat p \in [0,1]^k$ we have that: $\E_{(x,y) \sim \cD}[y | f(x) = \hat p] = \hat p$ --- in other words, calibrated forecasts are unbiased conditional on the value of the forecast. Then a simple consequence of calibration \cite{FV98,zhao2021calibrating,noarov2023high} is that \emph{for any expectation maximizing decision maker}, choosing the action that would maximize expected utility as if the forecast was correct is the optimal policy amongst all policies mapping forecasts to actions. Formally, if $f$ is calibrated, then applying the policy $BR_u(f(x)) = \arg\max_{a \in \cA} \E_{y \sim f(x)}[u(a, y)]$ obtains higher expected utility compared to applying any other policy $P:[0,1]^k\rightarrow \cA$ mapping forecasts to actions. In this sense, calibration is the right language for communicating uncertainty to \emph{expectation maximizing---i.e. risk neutral---agents}, and the right rule for such agents to ingest calibrated forecasts is to act as if they are correct specifications of the conditional label distribution. 

In contrast, we seek the right interface between predictions and actions for \emph{risk-averse agents}. Let \( a(\cdot) : \mathcal{X}\rightarrow \mathcal{A} \) be an action policy. We call \( \nu(\cdot) : \mathcal{X}\rightarrow \mathbb{R} \) a \emph{utility certificate} if it satisfies the following \emph{safety guarantee}:
\begin{align}\label{safety}
         \Pr [u(a(X), Y) \geq \nu(X)] \geq 1-\alpha.
\end{align}

In words, with probability at least \( 1-\alpha \), the utility of an agent following the policy \( a(x) \) is guaranteed to be at least \( \nu(x) \). Naturally, we aim to maximize the average value of the utility certificate \( \nu \) subject to satisfying the requirement in \eqref{safety} -- i.e., as the risk-averse agent, we  seek to maximize the average quantile of their utility, commonly referred to as the \emph{value at risk} in financial risk literature (see Section~\ref{Sec:problem} for details on the problem formulation). This will lead to the optimal trade-off between safety and utility, which is achieved by finding the action policy and utility certificate pair where the safety guarantee holds and the utility certificate is maximized. 

In practice, however, the true probability distribution that connects the actions to their utility values is unknown. Instead, the decision maker must rely on (uncertain) predictions to best balance the trade-off between safety and utility. The core challenge in this regard is to develop the right notion of uncertainty quantification for the predictions and optimal action policies based on  such uncertainty measures. 

We show that prediction sets are the right medium for communicating uncertainty to risk-averse decision makers who care about high-probability guarantees for their realized utility, i.e., the quantiles of their utility distribution as formulated in \eqref{safety}. In particular, we prove that finding optimal action policies whose utility is maximized subject to  satisfying the safety condition given in  \eqref{safety} is fundamentally equivalent to an optimal design of \emph{prediction sets} followed by a simple max-min action policy. This implies that prediction sets  constitute a sufficient statistic for designing safe action policies and thus encapsulate all the information necessary for risk-averse decision making. We then derive an explicit formulation for the optimal prediction sets based on which we develop a finite-sample algorithm to construct prediction sets with distribution-free safety guarantees. Put together, these results characterize the optimal interface between predictions and actions for risk-averse decision making as depicted in Figure~\ref{fig:pip}. 
In more detail:


\begin{figure}[t]
\centering
\includegraphics[width=\textwidth]{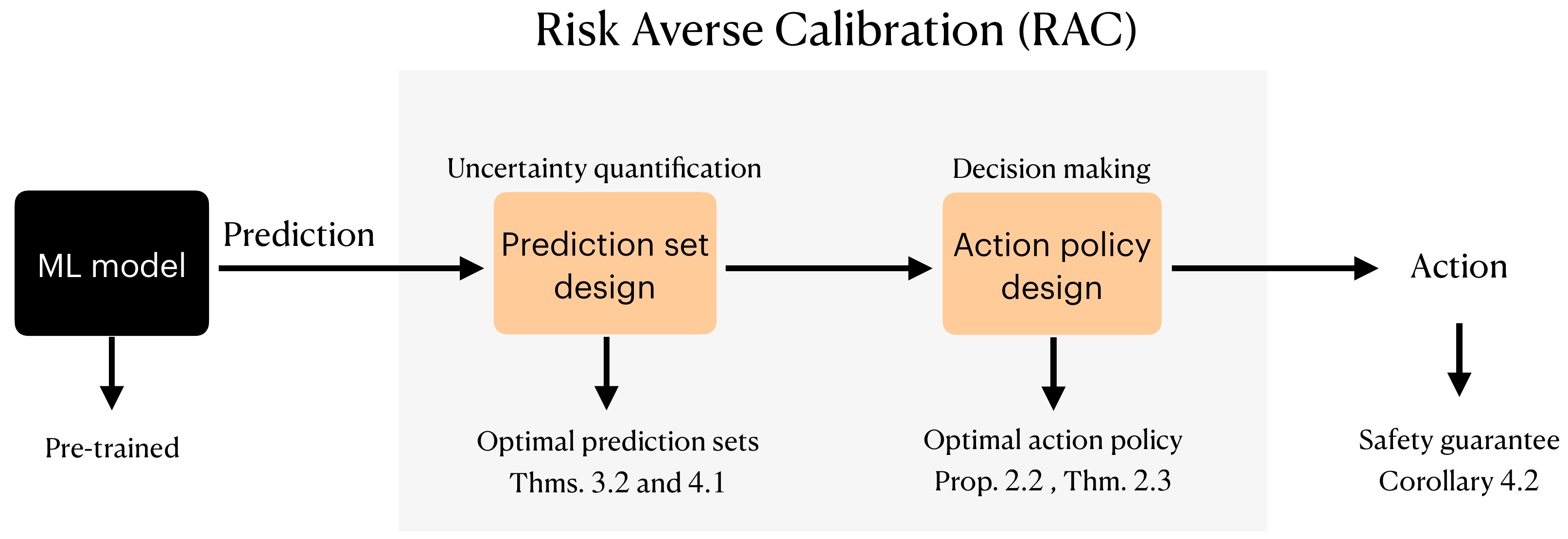}
\caption{RAC pipeline, an interface between prediction and action for high stakes applications.}
\label{fig:pip}
\end{figure}

\begin{enumerate}
    \item \textbf{Max-min decision rule.}  We show in Proposition \ref{maxmin_policy} that when given prediction sets $C(x)$ satisfying only a \emph{marginal} coverage guarantee, risk-averse decision makers should choose their action using a \emph{max-min decision rule}: i.e. they should choose the action $a$ that maximizes their utility $u(a,y)$ in the worst case over label outcomes $y \in C(x)$ contained within the prediction sets. In particular we show that this decision rule is minimax optimal over the set of data distributions for which the prediction rule's marginal coverage guarantee holds.
    \item \textbf{Prediction-set equivalence.} In Theorem~\ref{prediction_set_equivalence}, we show that the optimal pair of action policy and utility certificate can be derived from the max-min decision rule applied to a \emph{well-designed prediction set} with marginal coverage validity. This establishes that prediction sets are a sufficient statistic for safe decision making.
    
    
    
    \item \textbf{Optimal design of prediction sets.} With these results in hand, we focus on finding the prediction sets that optimize the downstream decision maker's target utility quantile, assuming that they are using the optimal max-min decision rule. 
    This optimization, which we call \emph{Risk Averse Conformal Prediction Optimization} is formulated in Section~\ref{Sec:equivalent}.
    We then use techniques from duality theory to derive an explicit formulation for the optimal prediction sets parametrized by a single-dimensional variable (see Theorem~\ref{strong_duality}). This characterization will be the basis for our finite-sample framework.
    \item \textbf{Finite-sample algorithm.} In Section \ref{sec:finite}, we introduce \emph{Risk Averse Calibration (RAC)},  a finite sample algorithm, which can exploit any black-box predictive model to derive action policies and utility certificates while providing a \emph{distribution-free} safety guarantee \eqref{safety}. This guarantee holds for any given utility function.
    
    \item \textbf{Experiments.} Finally, in two sets of experiments in \cref{sec:exp}, we compare the risk averse performance of RAC with several conformal prediction based methods, including \cite{cortes2024decision, romano2020classification, sadinle2019least}, as well as decision makers who act by best responding to a model's point prediction. We demonstrate the superior performance of RAC in handling the trade-off between safety and utility, meaning that our method delivers higher utility at every user-specified target quantile of the utility distribution. An Implementation of RAC can be accessed at the following link: 
\url{https://github.com/shayankiyani98/Risk-Averse-Calibration}.
\end{enumerate}

\subsection{Related Work}
The foundational idea of prediction sets can be traced back to early studies by \citet{Wilks1941, Wald1943, scheffe1945non, tukey1947non}. The initial concepts of conformal prediction (CP) were introduced in \citet{saunders1999transduction, vovk1999machine,vovk2005algorithmic}. With the advancement of machine learning, conformal prediction has become a widely adopted framework for constructing prediction sets \citep{Vovk2013,papadopoulos2002inductive,lei2018distribution,romano2020classification,romano2019conformalized,park2021pac, angelopoulos2020uncertainty}. There has been a growing body of work aiming to adapt conformal prediction methods for a range of decision-making problems. In the following, we will discuss the ones relevant to the present work. 


\textbf{Risk Control.} A growing line of research extends CP beyond coverage constraints to control more general risk measures \citep{lindemann2023safe, angelopoulos2022conformal, angelopoulos2021learn, cortes2024decision, lekeufack2024conformal, zecchin2024adaptive, blot2024automatically, zecchin2024localized}. In particular, \citet{angelopoulos2022conformal} propose conformal risk control for monotone risk measures over prediction sets, and \citet{cortes2024decision} extend this by constructing sets that satisfy coverage while achieving low risk. However,  these works  do not explicitly discuss which \emph{actions} their sets should inform or how to design these sets to best serve the decision maker.  Moreover, \citet{lindemann2023safe} applies conformal prediction to safe planning, and  \citet{lekeufack2024conformal} focuses on decisions parameterized by a single scalar, calibrated to control risk. However,  these works restrict their action policy to a \emph{predefined} low-dimensional family, leaving open the question of how to \emph{jointly} optimize over policy design and uncertainty quantification for risk-averse utility. 

In this paper, we fill this gap by addressing three core questions for a risk-averse decision maker: (1)~\emph{What is the correct notion of uncertainty quantification?} We prove that prediction sets are optimal for high-stakes decisions. (2)~\emph{How can we design these optimal sets?} We provide an exact population-level characterization and a distribution-free, finite-sample construction. (3)~\emph{What is the optimal policy given these sets?} We show that a simple max--min rule is optimal for risk-averse utility. In Section~\ref{sec:exp}, we  implement the most recent
approach in this direction \citep{cortes2024decision}, and demonstrate that our framework yields significantly more effective action policies.

On top of the fundamental differences we mentioned, there are also technical differences. After proving the equivalence of the risk-averse objective defined in Section \ref{Sec:fundamentals} to the prediction set optimization called RA-CPO in Section \ref{Sec:equivalent}, one might think we can define a risk function of the form $l(C) = - \max_{a \in \mathcal{A}}\; \min_{y \in C(x)} u(a, y)$, and then apply risk controlling methods to control this risk. However, controlling this risk alone is meaningless, as it is always possible to control the risk by outputting trivial sets. Hence, the risk should be controlled combined with coverage guarantees. The only risk controlling framework that additionally allows for a coverage constraint is the work of \citet{cortes2024decision}, which we compare our performance with  in  Section \ref{sec:exp}, and show our superior performance in handling the safety utility trade-off. Furthermore, the defined loss function $l$ for a generic utility function $u$, lacks any (approximate) separability property or sub-modularity, which are essential for algorithmic development of \citet{cortes2024decision}. We, however, work directly with the max-min objective and do not rely on any assumptions. For readers familiar with nested conformal prediction \citep{gupta2022nested}, perhaps another way to elaborate on this important technical difference is to look at Section \ref{Sec:optimalsets}, where in Theorem \ref{strong_duality}, we characterize the optimal prediction sets over the population. It is clear then that the optimal sets do \emph{not} necessarily form a nested sequence of sets as we sweep the miscoverage threshold $\alpha$. This is in contrast to when we want to find optimal sets corresponding to minimum average prediction set size (or any other separable objective). There, the optimal characterization is of the form $p(y|x) \geq q$ (or more generally of the form $s(x,y) \leq q$ for some score function $s$), where $q$ is tuned to satisfy the marginal coverage constraint \citep{lei2013distribution,sadinle2019least,kiyani2024length}. This distinction hints to the sub-optimality of the algorithms that  rely on monotonicity properties of the risk, e.g. thresholding a score function, in obtaining the best risk averse action policies and  safety~guarantee.

\textbf{Robust Optimization.} The max-min policy that we will discuss in Section~\ref{sec:maxmin} also naturally arises at the intersection of uncertainty quantification and robust optimization \citep{patel2024conformal, johnstone2021conformal, chenreddy2024end, li2025data, yeh2024end, cao2024non, wang2023learning, lou2024estimation, patel2024non, elmachtoub2023estimate, lin2024conformal, chan2024conformal, chan2023inverse, chan2020inverse}. 
In robust optimization, decision-making under uncertainty is typically formulated as a minimax problem, where an optimal decision is sought against worst-case realizations within an uncertainty set. Despite a structural resemblance of these works to our framework in that they involve optimization over an uncertainty set,  their scope and objectives have some fundamental differences from ours. We fix any black-box predictive model and any utility function, and in contrast to existing approaches, we \emph{jointly} characterize the optimal notions of uncertainty quantification and action policy. Specifically, we ask: (1) What is the appropriate uncertainty quantification for risk-averse decision makers? We answer that prediction sets are optimal for achieving high-probability utility guarantees. (2) How should these prediction sets be optimally constructed? We provide a distribution-free, finite-sample construction that characterizes the optimal sets. (3) What is the optimal decision policy given these sets? We prove that the max-min rule is provably optimal for risk averse agents. In doing so, our Risk-Averse Calibration (RAC) method offers a principled alternative to uncertainty sets based on heuristic conformity score designs, thereby contributing to the growing intersection of conformal prediction and robust optimization. Additionally, on a more technical note, in Section \ref{Sec:optimalsets}, we show that the optimal prediction sets that lead to optimal safe action policies when used in tandem with the max-min rule do \emph{not} necessarily take the form of thresholding a score function (i.e., $s(x,y) \leq q$ for some score function $s$). There, we characterize an alternative form that, in fact, captures the optimal prediction sets in the context of risk-averse decision-making. That is to say, our results hint to a principled alternative to conventional score-based prediction sets in the pipeline of robust optimization to avoid suboptimality.

\textbf{Further Related Work.} The potential connection of CP ideas to decision making has also been explored in \citet{vovk2018conformal}, from the point of view of conformal predictive distributions. Conformal predictive distributions produce calibrated distributions rather than prediction sets--see e.g. \cite{vovk2017nonparametric, vovk2018cross, vovk2020computationally}. Therefore, they are best to be compared with calibrated forecasts as the methodologies developed in \citet{vovk2018conformal} are also targeting expectation maximizer--i.e. risk neutral-- agents. We, however, focus on risk averse decision making using prediction sets. In particular, we show that prediction sets are a sufficient statistic for risk averse agents that aim to optimize their value at risk. 

Although our primary aim is to develop a general framework to construct prediction sets for  high-stakes decision-making, we note that conformal prediction sets have  been explored in a wide range of specific applications and domains of high-stakes nature. For instance, CP methods have been adapted and used in medical tasks \cite{banerji2023clinical}, power and energy systems \cite{renkema2024conformal}, formal verification and control \cite{lindemann2024formal},  chance-constrained optimization \cite{zhao2024conformal}, and more generally \citet{sun2024conformal, ramalingam2024uncertainty,kiyani2024conformal,straitouri2023improving, vishwakarma2024improving, kiyani2024length, vanderlaan2024selfcalibratingconformalprediction, noorani2024conformalriskminimizationvariance}. Our framework could potentially be extended to these domains, yet each may present additional, domain-specific challenges that lie beyond the scope of this work.

\section{Fundamentals of Risk Averse Decision Making}\label{Sec:fundamentals}
\subsection{Problem Formulation}\label{Sec:problem}
In this section, we will formalize the central objective of a risk averse decision maker.  We are given a  space of features (or covariates, contexts)  $\mathcal{X}$, a set of labels $\mathcal{Y}$, and a set of possible actions (or decisions) $\mathcal{A}$. We assume that the pair $(x,y) \in \mathcal{X} \times \mathcal{Y}$, i.e., the feature and the label, is generated according to a fixed but unknown distribution $\mathcal{D}$. Upon observing $x \in \mathcal{X}$, the decision maker will have to take an action $a \in \mathcal{A}$. Importantly, the decision maker does not observe the true label $y$. However, the utility of the decision maker will depend on both the chosen action $a$ and the label $y$, and is captured by a given utility function $u(\cdot, \cdot): \mathcal{A}\times\mathcal{Y}\rightarrow \R^+$.

\noindent In this paper, we will focus on \emph{risk-averse} decision making; I.e.,  the decision maker would like to choose its actions such that the obtained utility is \emph{guaranteed} to be large enough \emph{with high probability} over the randomness of the label. In other words, risk-averse behavior refers to a preference for actions that minimize the likelihood of low-utility outcomes, even if it means overlooking actions with potentially higher but uncertain utilities. To be more precise, consider a pre-specified risk tolerance threshold $\alpha$. For a given $x \in \mathcal{X}$, from the viewpoint of the risk-averse decision maker each action $a \in \mathcal{A}$ has value equal to: $\nu_\alpha(a; x):= \text{quantile}_{\alpha}[u(a, Y)\mid X = x]\,$\footnote{\(\text{quantile}_\alpha[Z] := \inf\{z \in \mathbb{R} \mid \mathbb{P}(Z \le z) \ge \alpha\}\) for any random variable \(Z\).},
where $Y$ is distributed according to $p(y|x)$. This objective is a standard risk measure and is known in the financial risk literature as the \emph{Value at Risk (VaR)} (see e.g. \cite{duffie1997overview}). The value at risk represents the largest value such that, if action $a$ is taken when facing $x$, then the obtained utility is guaranteed to be at least $v_\alpha(a; x)$ with probability $1-\alpha$. Consequently, the risk averse decision maker should choose the action that has the largest quantile value $v_\alpha(a; x)$, and the resulting best (risk-averse) utility will become: 
\begin{equation} \label{conditional-utility}
\nu_\alpha(x) = \max_{a \in \mathcal{A}} \;\nu_\alpha(a; x) := \max_{a \in \mathcal{A}} \,
\text{quantile}_{\alpha}\left[ u(a,Y) \mid X = x \right], \quad \forall x\in\mathcal{X}.
\end{equation}
The above \emph{risk-averse} utility  should be contrasted with the best \emph{expected} utility $\max_a \mathbb{E}[u(a,Y) | X = x]$. The latter leads to actions that maximize the average utility whereas the former aims to maximize the worst-case utility that can happen with probability $1-\alpha$. Hence the former will be
more risk averse at the cost of becoming more conservative. 

\noindent\textbf{Marginal Version.} The quantity in \eqref{conditional-utility} is a \emph{point-wise} or \emph{conditional} quantity; i.e. to find the best action according to \eqref{conditional-utility} the decision maker requires access to the conditional distribution  $p(y|x)$. In practice, such distributions are unknown and are often intractable when only a finite sample of the distribution is available. An analogous situation arises in conformal prediction (CP), where obtaining fully-conditional coverage guarantees  is known to be impossible from a finite sample of data \cite{pmlr-v25-vovk12, 2019arXiv190304684F}. Consequently, conformal prediction focuses on relaxed marginal (or ``group conditional'', which still marginalize over part of the distribution \cite{bastani2022practical,jung2023batch}) coverage guarantees which are statistically tractable. 

By analogy, we will now introduce the marginal version of \eqref{conditional-utility}. First we rewrite the objective. For a given $x \in \mathcal{X}$,  the value $v_\alpha(x)$ in \eqref{conditional-utility} can be equivalently written as follows
\[
 \begin{aligned}
 & \underset{a \in \mathcal{A}, \nu \in \mathbb{R}}{\text{Maximize}} & & \nu\\
 &  \text{subject to} & & \Pr [u(a, Y) \geq \nu \mid X = x] \geq 1 - \alpha.
 \end{aligned}
 \]
Let us examine the constraint in the above optimization more carefully. We are looking for action-value pairs $(a,\nu)$ such that we are guaranteed with probability at least $1-\alpha$ that, when taking action $a$, the resulting utility is at least $\nu$. Of course, to maximize utility, we should maximize over the choice of the action $a$ and the value $v$ which results in the above optimization.  Now, the risk-averse constraint in the above optimization has the following marginal counterpart: 
\begin{align}
         \Pr [u(a(X), Y) \geq \nu(X)] \geq 1-\alpha,
     \end{align}
 where the function $a(\cdot): \mathcal{X} \to \mathcal{A}$ is a decision-policy that\footnote{In this paper, we focus on deterministic action policies.}  maps features to actions such that it guarantees average  utility according to the function $\nu(\cdot): \mathcal{X} \to \mathbb{R} $ with probability at least $1-\alpha$, marginalized over $X$. Now, rather than optimizing over a single value for $a$ and $\nu$ for each $x$ separately, we jointly optimize over policies $a(\cdot)$ and value functions $\nu(\cdot)$\footnote{Here, note that since $\nu(x)$ is a utility function, its value can not be larger than the maximum achievable utility; i.e. $\nu(x) \leq u_{\rm max}:=\max_a \max_y u(a,y) $ for all $x \in \mathcal{X}$. } which map $\cX$ to actions and values respectively. This results in the following marginal version of the decision maker's optimization problem:
  \begin{mdframed}\label{RA-DPO}
\textbf{Risk Averse Decision Policy Optimization (RA-DPO):}
\[
\begin{aligned}
 \underset{a(\cdot), \,\nu(\cdot)}{\text{maximize}}\quad 
& \E_X\!\Bigl[\nu(X)\Bigr], 
\\[6pt]
\text{subject to}\quad
& \Pr [u(a(X), Y) \geq \nu(X)] \geq 1-\alpha,
\end{aligned}
\]
\end{mdframed}
\begin{remark}
Despite our primary focus on the marginal formulation of risk-averse optimization, the core arguments and methodologies presented in this work naturally extend to the more advanced setting of risk-averse optimization with group conditional validity. Specifically, consider a collection of arbitrary, pre-specified and potentially intersecting groups 
 $g_1, g_2, \cdots, g_m \subseteq \mathcal{X}$. The marginal constraint in RA-DPO can be generalized to a group-conditional form as follows:
$$
\Pr [u(a(X), Y) \geq \nu(X)\,\mid\,X\in g_i] \geq 1-\alpha, \quad \forall\; i \in [1, \cdots, m].
$$
These types of conditional constraints are known as group-conditional validity in the conformal prediction literature (see \cite{jung2021moment,bastani2022practical,jung2023batch,gibbs2023conformal}). This extension allows for a more fine-grained control over risk across different subpopulations or scenarios, which is crucial in applications where group-specific guarantees are important. Notably, all the theoretical results developed in \cref{Sec:equivalent} and \cref{sec:optimal_set} can be systematically generalized to accommodate these group-conditional constraints.
\end{remark}

\subsection{A Prediction Set Perspective}\label{sec:maxmin}
Recall that in our setting the (feature, label) pair is generated according to a distribution. The decision maker only observes the feature $x$ based on which it will choose its action $a$. However, the realized utility will depend on both the action  $a$ and the label $y$. The decision maker does not observe the label, but we assume that it has access to a predictor that provides predictions about the label $y$ given the input feature $x$. More specifically, we assume that the predictor will provide \emph{prediction sets} of the form $C(x) \subseteq \mathcal{Y}$, $x \in \mathcal{X}$, that are guaranteed to contain the true label with high probability. We assume that the prediction sets satisfy the \emph{marginal} coverage guarantee that is standard in conformal prediction, i.e., 
\begin{align}\label{marginal}
    \Pr_{(X, Y)}[Y \in C(X)] \geq 1 - \alpha.
\end{align}
Given this framework, two immediate questions arise: (i) Assuming the only information that the decision maker has about the true label is through the prediction sets, how should it choose its actions to maximize (risk-averse) utility? (ii) How should the prediction sets be designed to  not only be marginally valid according to \eqref{marginal} but also maximize the utility achieved by the decision maker?

We will  proceed with answering question (i) now, and will provide an answer to question (ii) in the subsequent sections. 
Assuming that the decision maker can only take actions based on the prediction sets -- i.e. it has no other information about the label distributions - then its optimal decision rule takes a simple and natural form. It will have to play the action $a$ that maximizes their utility $u(a,y)$ in the worst case over labels $y \in C(x)$. We denote this optimal risk-averse (RA) decision rule by $a_{\rm RA}: 2^\mathcal{Y} \to \mathcal{A}$, and the corresponding utility by $\nu_{\rm RA}: 2^\mathcal{Y} \to \mathbb{R}$:  
\begin{align}\label{maxmin_def}
a_{\rm RA}\bigl(C(x)\bigr) = \text{arg}\max_{a\in\mathcal{A}}\min_{y\in C(x)} u(a, y), \quad \nu_{\rm RA}\bigl(C(x)\bigr) = \max_{a\in\mathcal{A}}\min_{y\in C(x)} u(a, y).
\end{align}
We will show that this decision rule is minimax optimal over the set of all distributions that are consistent with the marginal guarantee \eqref{marginal}. Assume that the decision maker is given access to a set function $C: \mathcal{X}\rightarrow \{2^\mathcal{Y}\}$. Let us also define $\Omega$ as the set of all the  data distributions that are consistent with the marginal guarantee; i.e. the set of all distributions $p(x, y)$ over $(\mathcal{X}, \mathcal{Y})$ such that,
$\Pr_{(X, Y)\sim p(x, y)}[Y \in C(X)] \geq 1 - \alpha.$ Let $\pi(\cdot) : 2^\mathcal{Y}\rightarrow\mathcal{A}$ be a policy that takes as input the prediction set $C(x)$ and outputs an action. Aligned with RA-DPO, the value of policy $\pi$ with respect to a joint distribution $p(x,y)$ for the decision maker can then  be defined as:
\begin{align*}
    \nu^*(\pi, p) =  \;\underset{\nu(.)}{\text{Maximize}}\quad  &\underset{X\sim p(x)}{\E}\!\Bigl[\nu(X)\Bigr], 
\\[6pt]
\text{subject to}\quad
& \Pr_{X, Y \sim p(x, y)} [u(\pi (C(X)), Y) \geq \nu(X)] \geq 1-\alpha,
\end{align*}
We are now interested in the policy that is minimax optimal meaning that it can perform well with respect to the worst case distribution in $\Omega$. That is to say we want to find the policy $\pi^*$ that is the answer to:
\begin{align} \label{minimax}
    \underset{\pi}{\text{Maximize}}\;\underset{p \in \Omega}{\text{Minimize}} \quad\nu^*(\pi, p).
\end{align}
\begin{propo} \label{maxmin_policy}
    Assume $\alpha < 0.5$ and let $\pi^*(x)$ be the optimal solution to \eqref{minimax}. Then we have,
    \begin{align}
        \pi^*(x) = \arg\max_{a\in\mathcal{A}}\min_{y\in C(x)} u(a, y).
    \end{align}
\end{propo}

\noindent To summarize,  Proposition \ref{maxmin_policy} states that when the risk averse decision maker wants to make the decision based on  prediction sets $C(x)$, $x\in \mathcal{X}$, that contain the actual label with high probability, there is a simple, yet minimax optimal policy, $a_{\rm RA}\bigl(C(x)\bigr)$ that guarantees the minimum utility of $\nu_{\rm RA}\bigl(C(x)\bigr)$ with high probability. To this end, we turn our focus on how the prediction sets should be designed such that they would be the most useful for the decision maker among all the prediction sets that provide valid marginal guarantee. This is what we study in the next section.

\subsection{An Equivalent Formulation Through Prediction Sets }\label{Sec:equivalent}
In the previous section we argued that if the decision maker's knowledge about the label is only based on the prediction sets, then the (minimax) optimal policy $a_{\rm RA}$ and its associated value $\nu_{\rm RA}$ are given in \eqref{maxmin_def}. Hence, assuming that the decision maker is playing $a_{\rm RA}$, the prediction sets $C(x)$ should be designed to maximize the resulting utility of the decision maker while ensuring marginal coverage; I.e., the prediction sets $C(x)$ should be designed according to the following optimization:
\begin{mdframed}\label{Secondary}
\textbf{Risk Averse Conformal Prediction Optimization (RA-CPO):}
\[
\begin{aligned}
& \underset{C(.)}{\text{Maximize}} & & \E_X \biggl[\nu_{\rm RA}(C, X)\biggr] := \E_X \left[\max_{a\in\mathcal{A}}\min_{y\in C(X)} u(a, y)\right] \\
&  \text{subject to} & & \Pr [Y \in C(X)] \geq 1 - \alpha.
\end{aligned}
\]
\end{mdframed}
One might expect that the result of RA-CPO, i.e. optimizing the utility using prediction sets, would lead to a lower utility compared to the original optimization RA-DPO. This is because: (i) The policy given in \eqref{maxmin_def} is a specific policy designed to be valid even for the worst-case distribution for which the prediction sets are marginally valid (see Proposition \ref{maxmin_policy}). Hence, this policy could be overly conservative; (ii) In RA-DPO the optimal action and value functions are obtained assuming full information about the data distribution, whereas in RA-CPO we require that information must be filtered through a (properly designed) prediction set representation. One might expect a-priori that passing from the actual distribution to a lossy prediction set representation would discard information that is critical to finding the optimal policy. However, the following theorem shows, perhaps surprisingly, that  this is not the case; \emph{the optimal action policy for any distribution can be represented as a max-min rule over a prediction set.}

\begin{theorem}\label{prediction_set_equivalence}
    RA-DPO and RA-CPO are equivalent. In other words, from any optimal solution of RA-DPO, denoted by $(a^*(x), \nu^*(x))$, we can construct an optimal solution $C^*(x)$ to RA-CPO with the same utility, i.e., $\E_X \left[\nu_{\rm RA}\bigl(C^*(X)\bigr)\right] = \E_X\left[\nu^*(X)\right].$ Also, from any optimal solution of RA-CPO we can construct an optimal solution for RA-DPO with the same utility. 
\end{theorem}

\noindent\textbf{Implications.} Prediction sets are a fundamental object in risk averse decision making. In particular, the optimal strategy of a risk averse decision maker can be formulated as playing a max-min strategy over an optimized prediction set. To fully characterize such optimal policies, the first step
is to derive the optimal solution to RA-CPO. This is the subject of the study in the next section.


\section{The Optimal Prediction Sets} \label{Sec:optimalsets}
In this section, we characterize the optimal solution (i.e., prediction sets) for RA-CPO~\ref{Secondary} in terms of the conditional distribution \(p(y \mid x)\). Before that, let us summarize what we have done so far and give an overview of the steps we are going to take in this section.  We proved in the previous section  that RA-DPO \eqref{RA-DPO} and RA-CPO \eqref{Secondary} are equivalent. To solve the RA-CPO we will first introduce a reparametrization as in \eqref{opt_reparametrization}, which we then solve using techniques from duality theory. 

\noindent We now focus on three fundamental objects (denoted by bold symbols) that describe the optimal sets:
\begin{enumerate}
    \item A function \(\btheta(x,t)\in \mathbb{R}\), which represents the optimal (risk-averse) utility achievable under a conditional coverage assignment \(t \in [0,1]\) for a covariate $x \in \mathcal{X}$.
    \item A function \(\ba(x,t)\in \mathcal{A}\), which gives the action that attains the corresponding risk-averse utility level.
    \item A function \(t^*(x) \in [0,1]\) that optimally assigns conditional coverage probabilities for \(x \in \mathcal{X}\), ensuring a total coverage of \(1-\alpha\). Moreover, as we will show, \(t^*(x)\) can be characterized via a single scalar \(\beta\) and a     function \(\bg(x,\beta)\), which is derived from the conditional distribution and the utility values. 
\end{enumerate}



\noindent We begin by introducing the main notion that we use to relate optimal utility to coverage. We define the functions $\btheta: \mathcal{X} \times [0,1] \to \mathbb{R}$ and $\ba: \mathcal{X} \times [0,1] \to \mathcal{A}$ as follows: 
\begin{align} \label{theta}
     \btheta(x, t) = \max_{a \in \mathcal{A}}\text{quantile}_{1-t}\left[u(a,Y) \mid X = x\right], \quad
     \ba(x, t) = \arg\max_{a \in \mathcal{A}}\text{quantile}_{1-t}\left[u(a,Y) \mid X = x \right]
\end{align}

\begin{figure}[t]
\centering
\includegraphics[width=0.8\textwidth]{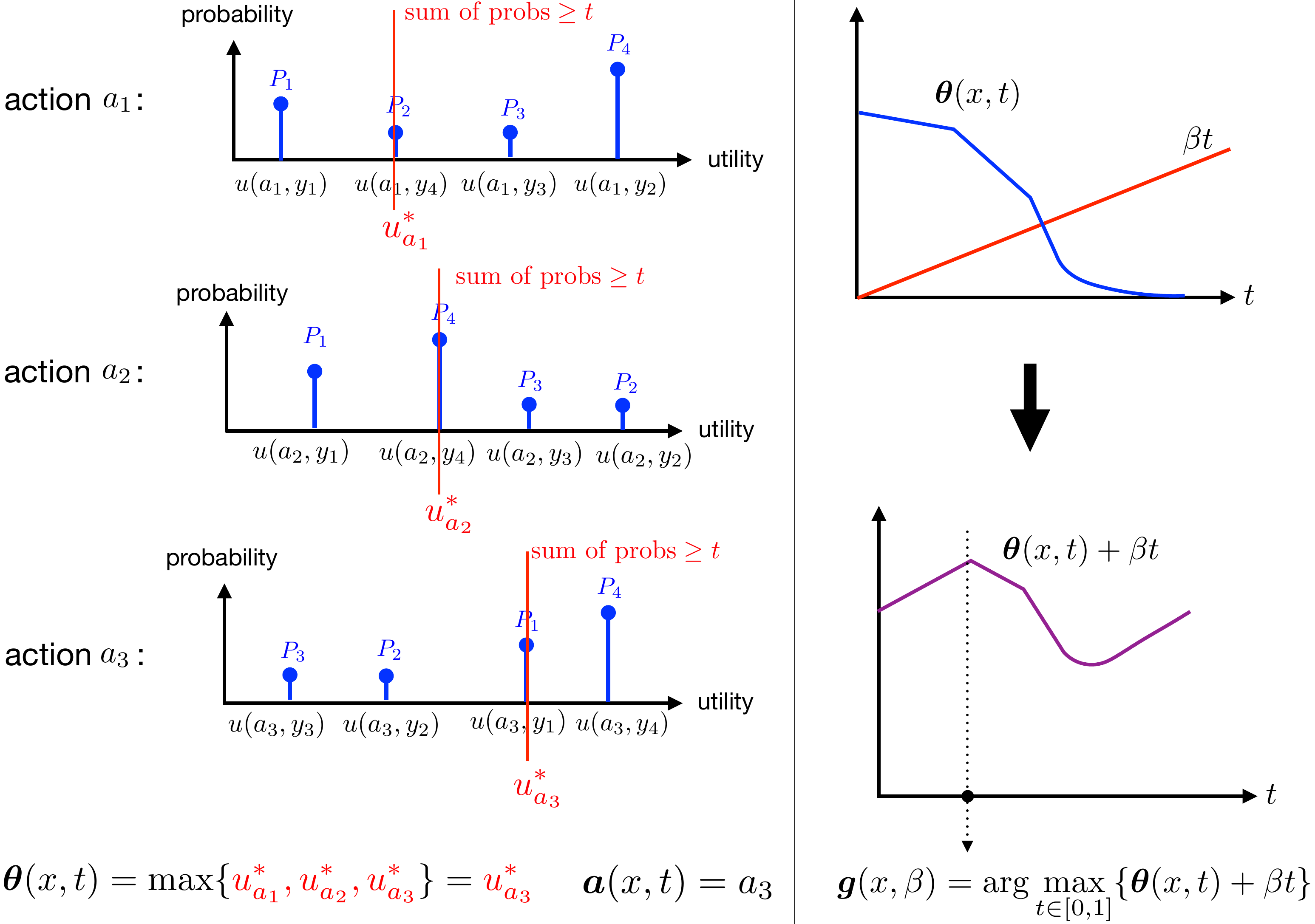}
\caption{\emph{Left}: Illustration of how the functions $\btheta$ and $\ba$ are computed for a given $x \in \mathcal{X}$ and $t \in [0,1]$. Here, we have three actions $\mathcal{A} = \{a_1, a_2, a_3\}$ and four labels $\mathcal{Y} = \{y_1, y_2, y_3, y_4\}$. We also let $P_i := p(y_i|x)$ denote the conditional probabilities. For each of the actions $a_j$, $j=1,2,3$, the value $u^*_{a_j}$ is the $(1-t)$-quantile of the random variable $u(a_j, Y)$. The value $\btheta(x,t)$ corresponds to the maximum of these quantiles among the actions, and $\ba(x,t)$ corresponds to the maximizing action.  \emph{Right}: Illustration of how the function $\bg(x, \beta)$ is obtained from $\btheta(x,t)$ for a given $x$.}
\label{fig:illustrate_theta}
\end{figure}

\noindent In words, given a feature $x \in \mathcal{X}$ and a probability coverage value $t \in [0,1]$, $\btheta(x,t)$ is computed as follows (see also Figure~\ref{fig:illustrate_theta}): For each action $a$, we first find the $(1-t)$-quantile of the random variable $u(a,Y)$ with $Y$ being distributed according to $p(y|x)$. This quantile value is the largest utility achievable with probability at least $t$ when we take action $a$. By maximizing such $(1-t)$-quantiles over the choice of the action $a$ we obtain $\btheta(x,t)$, and the maximizing action is denoted by $\ba(x,t)$.

\noindent Let us now explain how the function $\btheta(x,t)$ can play a role in finding an optimal solution for RA-CPO \eqref{Secondary}. Fix an instance $x$, and assume that we would like to assign coverage probability $t$ to $x$. For the specific instance $x$, we would like to construct a prediction set $C(x)$ with coverage at least $t$, i.e. $\text{Pr}(Y \in C(x)) \geq t$, where the probability is over the conditional distribution $p(y|x)$. We ask: How should $C(x)$ be designed to maximize the objective of RA-CPO? The following proposition provides the answer.

\begin{propo}
    \label{optimal_prediction_set}
    Fix an instance $x \in \mathcal{X}$ and a coverage value $t \in [0,1]$. Then, among all the sets $C \subseteq \mathcal{Y}$ that have coverage at least $t$, i.e. ${\rm{Pr}}(Y \in C | X=x) \geq t$, the following set  has the largest risk-averse utility value $\nu_{\rm RA}(C) = \max_{a \in \mathcal{A}} \min_{y \in C} u(a,y)$:
        \begin{align} \label{C(x,t)}
        C(x,t) = \biggl\{
        y\in\mathcal{Y}:\,\,u(\ba(x,t), y)  \geq 
        \btheta(x, t)
        \biggr\},
    \end{align}
    Further,  we have $\nu_{\rm RA}(C(x,t)) = \btheta(x,t)$.

 \label{sec:optimal_set}
\end{propo}
\noindent Given the above proposition, the optimal set for RA-CPO \ref{Secondary} can be obtained  based on the following re-parametrization in terms of the coverage probabilities that we assign to each $x \in \mathcal{X}$. In order to satisfy the marginal constraint of RA-CPO, we will need to assign to each $x$, a coverage value $t(x)$ such that $\mathbb{E}_X[t(X)] \geq 1-\alpha$. From the above proposition, we know that if an instance $x$ is assigned with $t$ units of (probability) coverage, then it can add the maximum amount of $\btheta(x, t)$ to the objective and its corresponding prediction set, which is optimal given $t$ units of coverage assigned to $x$, is given in \eqref{C(x,t)}.  
Hence, to find the optimal prediction sets we will need to find the assignment $t(x)$ which optimally distributes the $(1-\alpha)$ units of probability over the feature space $\mathcal{X}$, such that the expected utility is optimized.  Formally, RA-CPO reduces to choosing a function \(t:\mathcal{X}\to[0,1]\) that satisfies the total coverage \(\E_X[t(X)] \ge 1-\alpha\) and maximizes \(\E_X[\btheta(X,t(X))]\). This step is captured by the following equivalent reformulation of RA-CPO:
\begin{equation} \label{opt_reparametrization}
\begin{aligned}
& \underset{t: \mathcal{X} \to [0,1]}{\text{maximize}} & & \E_X \bigl[\btheta(X, t(X))\bigr] \\
& \text{subject to:} & & \E_X\bigl[t(X)\bigr] \,\ge\, 1-\alpha.
\end{aligned}
\end{equation}

\noindent 
Once the optimal solution $t^*(x)$ to the above re-parametrization of RA-CPO is found, then the optimal policy/actions, denoted by $a^*(x) = \ba(x,t^*(x))$,  are derived according to \eqref{theta}, and the optimal prediction set is given by:
    \begin{align} \label{C^*}
        C^*(x) = \biggl\{
        y\in\mathcal{Y}:\,\,u(a^*(x), y)  \geq 
        \btheta(x, t^*(x))
        \biggr\}.
    \end{align}
\noindent Using tools from duality theory, we can show that the above optimization problem admits a solution with a simple ``one-dimensional'' structure in terms of scalar parameter $\beta \in \mathbb{R}$ and an assignment function  $\bg: \mathcal{X} \times \mathbb{R} \to [0,1]$ defined as
\footnote{For simplicity, we assume in this section that the maximizer of $\theta(x,s) + \beta s$ is unique with probability $1$ for any $\beta > 0$. In Appendix~\ref{app:proof_duality}, we will derive a general perturbation-based rule when the maximizer is not unique.}
    \begin{align}\label{g-beta}
        \bg(x, \beta) = \arg\max_{s \in [0,1]} \bigl\{\btheta(x,s) +  \beta s  \bigr\}.
    \end{align}
\noindent An illustration of the function $\bg(x,\beta)$ is provided in Figure~\ref{fig:illustrate_theta}. One can observe that $\bg(x,\cdot)$ is connected to the convex-conjugate transform of the function $\btheta(x, \cdot)$. 

\begin{theorem}\label{strong_duality}
    Assume that the marginal distribution of $X$, $\mathcal{P}_X$, is continuous. Then, there exists a $\beta^* \geq 0$ such that the function
    \begin{align*}
       t^*(x) = \bg(x,\beta^*)
    \end{align*}
    is an optimal solution of \eqref{opt_reparametrization}. Consequently,  the optimal prediction sets for RA-CPO \eqref{Secondary} are obtained using $t^*(x)$ from \eqref{C^*}. Further, the value of $\beta^*$ is a solution to the following equation in terms of the scalar $\beta$: $\mathbb{E}_X[\bg(X,\beta)] = 1 - \alpha$.  
\end{theorem}

\noindent The main implication of the above theorem is that it provides a simple characterization of the optimal sets given access to the data distribution: (i) Find the scalar $\beta^*$ that satisfies $\mathbb{E}_X[\bg(X,\beta^*)] = 1 - \alpha$; (ii) For each $x \in \mathcal{X}$ compute $t^*(x):=\bg(x, \beta^*)$ from \eqref{g-beta}; (iii) The optimal prediction set for $x$, $C^*(x)$, is then given by \eqref{C^*}.

\noindent
As we will see in the next section, the one-dimensional characterization via \(\beta\) is particularly convenient when we only have access to approximate conditional probabilities. By replacing $p(y| x)$ with a suitable approximation in all the definitions in this section, we can still apply the construction in Theorem~\ref{strong_duality}—namely, search for a  $\beta$ that ensures the corresponding prediction sets maintain valid coverage. This simple scalar calibration then yields prediction sets whose risk-averse utility are improved (and eventually becomes optimal) as the quality of the approximated probabilities improves.

\section{The Main Algorithm: Risk Averse Calibration (RAC)}\label{sec:finite}
So far, we have shown that RA-DPO \eqref{RA-DPO}, the primary problem that a risk averse agent cares about, is equivalent to RA-CPO \eqref{Secondary}, an optimization problem over prediction sets. In Proposition \ref{sec:optimal_set} and Theorem \ref{strong_duality}, we derived the structure of the optimal prediction sets for the RA-CPO problem. These sets are defined by the following  functions: (i) $\btheta(x,t)$ given in \cref{theta}, which fundamentally relates coverage to utility; (ii) $\ba(x,t)$ given in \cref{theta},  which is the corresponding action that provides the desired utility; and (iii) the assignment function $\bg(x,\beta)$ introduced in \cref{g-beta}. These quantities are defined based on the true conditional distribution which is often unknown in practice.

In this section, we consider the finite-sample setting in which we assume access to a set of calibration samples $\{(X_i, Y_i)\}_{i=1}^n$, as well as a predictive model, $f: \mathcal{X}\rightarrow \Delta_\mathcal{Y}$, which assigns to each $x \in \mathcal{X}$ a $|\mathcal{Y}|$-dimensional probability vector. Here, the output of $f$ for an input $x$, denoted by $f_x$, should be thought of as the approximate probabilities that a forecaster provides for the labels given the input $x$. 
For example, $f_x$ can be the (softmax) output of a pre-trained model that predicts label $y$ for the input $x$. We use the notation $f_x(y)$ to denote the probability assigned to label $y$ for the input $x$. Now, we aim to develop a finite sample algorithm that connects predictions to actions that can exploit any black-box pre-trained predictive model.

Using the model $f$, we will estimate the functions $\btheta$, $\ba$, and $\bg$, defined in \eqref{theta} and \eqref{g-beta}, by substituting the true conditional probabilities with their estimated counterparts obtained via $f$. Concretely,
\begin{align} \label{theta_hat}
     \hat{\btheta}(x, t) = \max_{a \in \mathcal{A}}\text{quantile}_{1-t}\left[u(a,Y) \mid Y \sim f_x\right], \quad
     \hat{\ba}(x, t) = \arg\max_{a \in \mathcal{A}}\text{quantile}_{1-t}\left[u(a,Y) \mid Y \sim f_x \right]
\end{align}
and
\begin{align}\label{g_hat}
    \hat{\bg}(x, \beta) = \arg\max_{s \in [0,1]} \left\{ \hbtheta(x,s) +  \beta s  \right\}.
\end{align}
From the result of Theorem~\ref{strong_duality} we know that the optimal prediction sets admit a ``one-dimensional'' structure in terms of the scalar parameter $\beta \in \mathbb{R}$, and the optimal coverage assignment is derived using the function $\bg(x, \beta)$. Hence, to simplify notation, we analogously define
\begin{align*}
\hat{\btheta}(x, \beta) \;:=\; \hat{\btheta}\!\bigl(x,\,\hat{\bg}(x, \beta)\bigr), \quad \hat{\ba}(x,\beta) \;:=\; \hat{\ba}\!\bigl(x,\,\hat{\bg}(x, \beta)\bigr).
\end{align*}
Following \eqref{C^*}, the prediction sets take the form
\begin{align*}
  \hat{C}(x; \beta) = \biggl\{
        y\in\mathcal{Y}:\,\,u\left(\hat{\ba}(x,\beta), y \right)  \geq 
        \hat{\btheta}(x, \beta)
        \biggr\}.
\end{align*}
All of the defined functions, $\hat{\btheta}, \hat{\ba}, \hat{\bg},$ and $\hat{C}$ are easily computable as they are parametrized by a one-dimensional parameter $\beta$. We can now present Risk Averse Calibration (RAC) in algorithm \ref{alg}, which gives a simple finite sample method to calibrate the parameter $\beta$ with distribution free guarantees.

\begin{algorithm}[t!]
\caption{Risk Averse Calibration (RAC)}\label{alg}
\begin{algorithmic}[1]
\State \textbf{Input:} Miscoverage level $\alpha$, Calibration samples $\{(X_i, Y_i)\}_{i=1}^n$, Test covariate $X_{\rm test}$.
\State \textbf{for each $y \in \mathcal{Y}$:} 
\[
\begin{aligned}
\hat{\beta}_y = \underset{{\beta\in \mathbb{R}}}{\text{argmin}}\, \beta \quad\text{subject to:}\quad \frac{1}{n+1}\bigl\{\sum_{i=1}^n \textbf{}[Y_i \in \hat{C}(X_i; \beta)] + \mathbf{1}[y \in \hat{C}(X_{\rm test}; \beta)]\bigr\} \ge 1-\alpha.
\end{aligned}
\]
\State \textbf{Output:}
\[
\begin{aligned}
C_{\rm RAC}(X_{\rm test}) = \bigl\{y \in \mathcal{Y}\,\mid\, y\in \hat{C}(X_{\rm test}; \hat{\beta_y})\bigr\}.
\end{aligned}
\]
\end{algorithmic}
\end{algorithm}


\begin{theorem}\label{coverage_theorem}
    Assume that the calibration samples $\{(X_i, Y_i)\}_{i=1}^n$ and $(X_{\rm test}, Y_{\rm test})$ are exchangeable. Then, we have
    \begin{align*}
        \Pr \big[Y_{\rm test}\in C_{\rm RAC}(X_{\rm test})\big] \ge 1 - \alpha,
    \end{align*}
over the randomness of the test and calibration data.
\end{theorem}
\noindent Put it differently, Theorem \ref{coverage_theorem} states that the prediction sets constructed by RAC have the so-called property of distribution-free coverage guarantee. 

\noindent  Recalling the definitions \eqref{maxmin_def}, we can now state the following corollary.
\begin{corollary}\label{criteria_1_cor}
    Assume that the calibration samples $\{(X_i, Y_i)\}_{i=1}^n$ and $(X_{\rm test}, Y_{\rm test})$ are exchangeable. We then have
    \begin{align*}
        \Pr \big[u\bigl(a_{\rm RA}\left(C_{\rm RAC}(X_{\rm test})\right),\,Y_{\rm test})\ge \nu_{\rm RA}(C_{\rm RAC}(X_{\rm test})\bigr)\big] \ge 1 - \alpha.
    \end{align*}
\end{corollary}
Putting the pieces together, Corollary~\ref{criteria_1_cor} ensures that a simple max-min decision policy over RAC-constructed prediction sets provides a pair of \emph{action policy} and \emph{utility certificate}, namely $a_{\rm RA}(C_{\rm RAC}(X_{\rm test}))$ and $\nu_{\rm RA}(C_{\rm RAC}(X_{\rm test}))$, satisfying a distribution free safety guarantee according to \eqref{safety}. Moreover, Theorem~\ref{strong_duality} highlights RAC’s practical relevance in terms of exploiting the predictive model. Specifically, RAC’s utility performance depends on the quality of the predictive model $f$: if $f$ closely estimates the true conditional probabilities, then the model-based definitions in \eqref{theta_hat} and \eqref{g_hat} approximate their true counterparts in \eqref{theta} and \eqref{g-beta}, ensuring that RAC-informed decisions align closely with the optimal ones, as guaranteed by Theorem~\ref{strong_duality}.

\section{Experiments}\label{sec:exp}

In this section, given a pre-trained model, \( f(\cdot) \), which assigns probability \( f_x(y) \) to input-label pair \( (x, y) \), we compare RAC with two groups of baselines:

\noindent \textbf{Calibration + Best-Response.} We calibrate the model on the calibration data using a strengthened version of \emph{decision calibration} \cite{zhao2021calibrating}, specifically the variant from \cite{noarov2023high}, which provides \emph{swap regret} bounds. We then apply the \emph{best-response} policy:
$
\text{best-response}(x) \;=\; \arg\max_{a \in \mathcal{A}} \;\E_{y \sim f_x(y)} \bigl[u(a, y)\bigr].
$ 
The primary purpose of implementing this baseline is to highlight the consequences of fully trusting the predictive model. While we expect the best-response policy over a calibrated model to achieve higher average utility at test time (see the Section \ref{intro}), it is likely to make more frequent critical mistakes compared to our method.

\noindent \textbf{Conformal Prediction + Max-Min.} We construct \((1-\alpha)\)-valid prediction sets using split conformal prediction with three different scoring rules. The decision policy then applies the max-min rule from Section~\ref{Sec:fundamentals}:  
$
a_{\rm RA}(C(x)) \;=\; \arg\max_{a \in \mathcal{A}} \; \min_{y \in C(x)} \; u(a, y),
$
which we proved is the optimal strategy when deciding based on prediction sets in Section~\ref{Sec:fundamentals}. The three scores are: 

\begin{itemize}
    \item \textbf{score-1} \cite{sadinle2019least}: \(1 - f_x(y)\),
    
    \item \textbf{score-2} \cite{romano2020classification}: 
    $\underset{y':\, f_{x}(y') > f_x(y)}{\sum f_{x}(y')}$,
    
    \item \textbf{score-3} \cite{cortes2024decision}: a greedy scoring rule tailored to the max-min policy.
   
\end{itemize}
By varying \(\alpha\), we can control the degree of conservativeness, trading off average utility against the avoidance of catastrophic errors.  We compare in terms of safety and~utility using the following metrics:

\begin{itemize}
    \item 
    \textbf{(a) Average realized max-min value}: The test-time mean of the worst-case utility across the prediction sets (i.e., the average of \(\nu_{\rm RA}\) in \eqref{maxmin_def}). 
    \item 
    \textbf{(b) Fraction of critical mistakes}: For samples with a critical ground-truth label, we report the fraction of cases in which each method chooses the \emph{worst} action in test-time.
    \item 
    \textbf{(c) Average realized utility}: The empirical mean of the realized utilities across all test samples.
    \item 
    \textbf{(d) Realized miscoverage}: The fraction of test samples for which the true label is not in the prediction~set.
\end{itemize}

\begin{figure*}[t]
    \centering
    \begin{subfigure}[t]{0.49\textwidth}
        \centering
        \includegraphics[width=\linewidth]{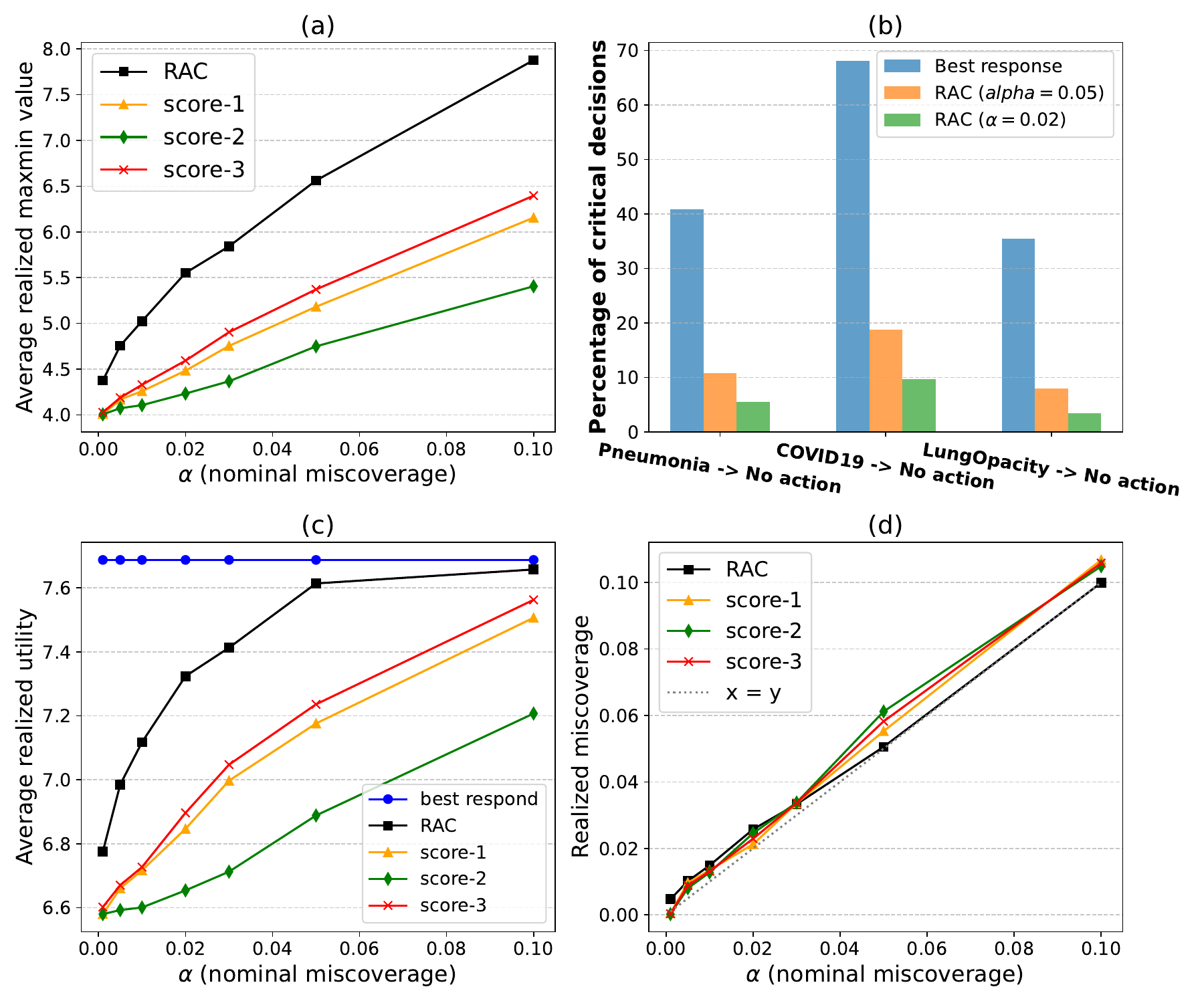}
        \caption*{\textbf{Medical Diagnosis Experiment}}
    \end{subfigure}
    \hfill
    \begin{subfigure}[t]{0.49\textwidth}
        \centering
        \includegraphics[width=\linewidth]{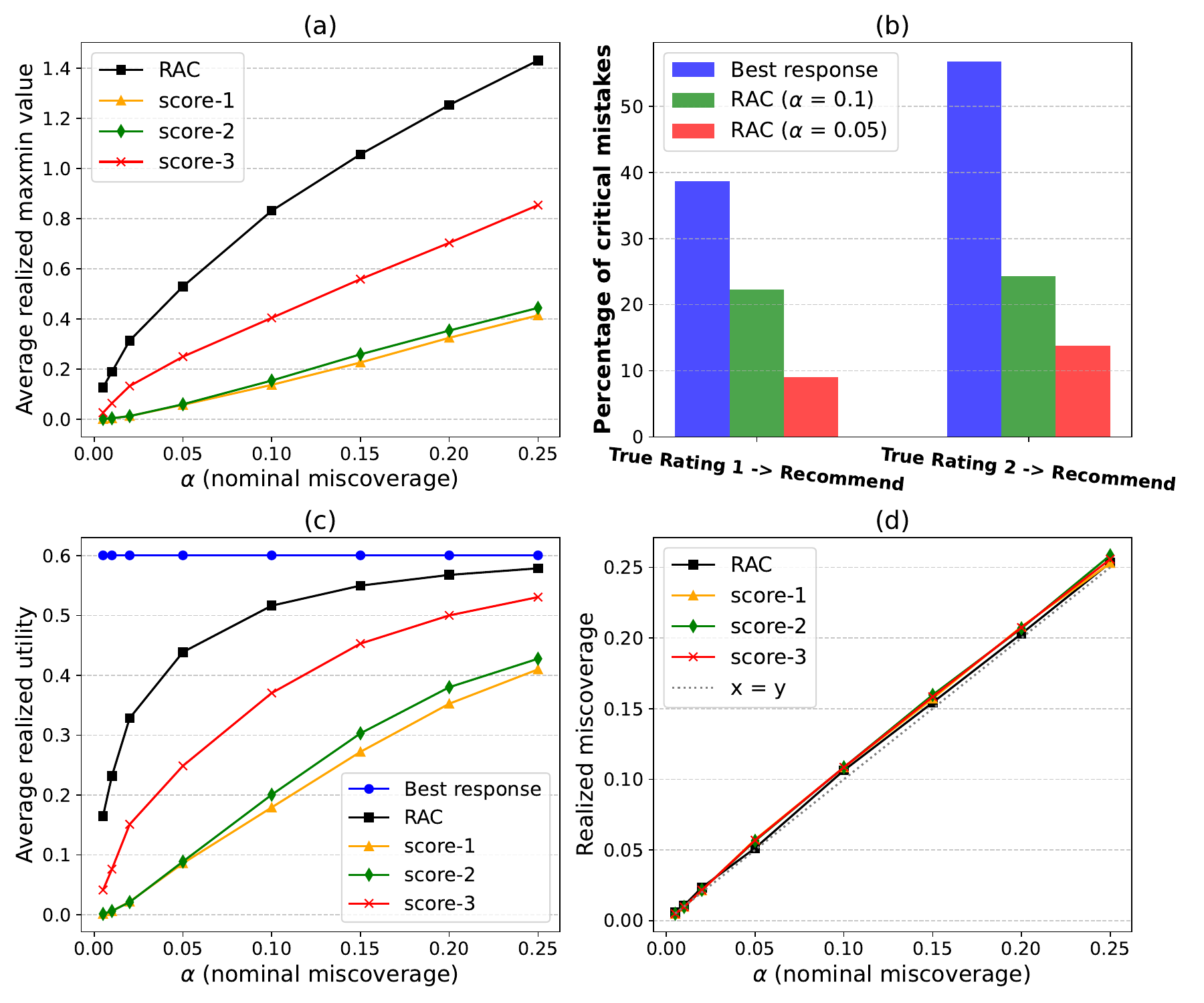}
        \caption*{\textbf{MovieLens Recommendation Experiment}}
    \end{subfigure}
    \caption{Results from two experiments. (a) Average realized max-min value as a function of $\alpha$. (b) Fraction of wrong critical decisions: in medical diagnosis, severe omission of appropriate care (e.g., failing to act on pneumonia or COVID-19 cases); in MovieLens, the percentage of movies rated 1 or 2 that were incorrectly recommended. (c) Average realized utility. (d) Realized miscoverage.}
    \label{fig:resultsplots_combined}
\end{figure*}



\subsection{Medical Diagnosis}\label{exp:med}
In this experiment, we explore decision making in medical diagnosis and treatment as a \textbf{risk-sensitive} application. Specifically, we use the \textit{COVID-19 Radiography Database}~\cite{chowdhury2020can, rahman2021exploring}, which contains chest X-ray images of four categories: \emph{Normal}, \emph{Pneumonia}, \emph{COVID-19}, and \emph{Lung Opacity}. We randomly split the data into \textbf{training} (70\%), \textbf{calibration} (10\%), and \textbf{test} (20\%) sets. In this experiment, we use the \texttt{Inception\_v3} architecture \cite{szegedy2015going, szegedy2016rethinking}, developed by google, a convolutional neural network known for its inception modules that employ multiple filter sizes in parallel. We initialize the model with ImageNet-pretrained weights and then \emph{fine-tune} it on our dataset, meaning we retrain the higher layers while preserving much of the earlier-layer feature representations. 

To capture clinical priorities, we employ the \textbf{utility matrix} in Table~\ref{tab:utility}, which maps each true condition (row) to a set of actions (column). Although we use the specific matrix below, our setup can accommodate any alternative choice of the utilities. (Further details on the \emph{AI-assisted} construction appear in the Appendix \ref{utility_AI}) All the baselines then will be calibrated to connect model's predictions to these four actions.

\begin{table}[h]
\centering
\begin{tabular}{lcccc}
\toprule
\textbf{True Label} & \textbf{No Action} & \textbf{Antibiotics} & \textbf{Quarantine} & \textbf{Additional Testing} \\
\midrule
Normal (0)        & 10 & 2 & 2 & 4 \\
Pneumonia (1)     & 0  & 10 & 3 & 7 \\
COVID-19 (2)      & 0  & 3  & 10 & 8 \\
Lung Opacity (3)  & 1  & 4  & 4  & 10 \\
\bottomrule
\end{tabular}
\caption{Utility matrix for the four-class chest X-ray task. Rows denote true conditions, columns represent actions.}
\label{tab:utility}
\end{table}

After training, we vary the nominal miscoverage parameter \(\alpha\) during calibration to study its impact on performance. As shown in Figure~\ref{fig:resultsplots_combined}(a), our method achieves the best trade-off curve among baselines, providing higher worst-case utilities for every nominal \(\alpha\). Equivalently, it offers stronger \emph{utility certificates} at each high-probability threshold. In Figure~\ref{fig:resultsplots_combined}(c), it also consistently outperforms other prediction set-based methods in terms of \emph{average utility}. 

As we expected, the best-response method over a calibrated model attains the highest overall average utility, however, Figure~\ref{fig:resultsplots_combined}(b) highlights its susceptibility to critical mistakes. For example, in COVID-19 cases, best-response chooses \emph{no action} over 60\% of the time, recommending a wrong treatment on a large fraction of patients with COVID-19. Our risk-averse policy (RAC) drive this error rate below 10\% (at \(\alpha=0.02\)), incurring only a modest (under 5\%) drop in average utility. Finally, Figure~\ref{fig:resultsplots_combined}(d) confirms that all prediction-set-based baselines achieve their target miscoverage levels, ensuring the associated high-probability utility guarantees remain statistically~valid.
\subsection{Recommender Systems}
We next consider a risk-sensitive recommendation scenario using the \textit{MovieLens} dataset. Each data point is a user--movie pair \(\bigl(x=(\text{user features},\text{movie features}), y\bigr)\), where the label \(y \in \{1,2,3,4,5\}\) is the user's rating. We split the data into \textbf{training} (80\%), \textbf{calibration} (10\%), and \textbf{test} (10\%), and train a neural network classifier \(f\) (details in the Appendix) to estimate the probability distribution \(f_y(x)\).

At test time, the policy must decide whether to \emph{recommend} or \emph{not recommend} a movie. We use the  \textbf{utility function} in Table~\ref{tab:utility_movie}: if a movie with true rating \(y\) is recommended, the utility is \(y - 3\), while not recommending yields~0. 

\begin{table}[h]
\centering
\small
\begin{tabular}{lccccc}
\toprule
\textbf{Action} & 1 & 2 & 3 & 4 & 5\\
\midrule
Not Rec & 0 & 0 & 0 & 0 & 0 \\
Rec     & -2 & -1 & 0 & +1 & +2 \\
\bottomrule
\end{tabular}
\caption{Utility matrix for the MovieLens recommendation task.}

\label{tab:utility_movie}
\end{table}

We vary the nominal miscoverage \(\alpha\) during calibration and measure performance on test data. As shown in Figure~\ref{fig:resultsplots_combined}(a), our method achieves the best trade-off among baselines, offering stronger \emph{utility certificates} (worst-case utility) at all \(\alpha\) levels. Figure~\ref{fig:resultsplots_combined}(c) also shows that our approach outperforms other CP-based methods in \emph{average utility}.

Although the best-response method achieves the highest overall average utility, Figure~\ref{fig:resultsplots_combined}(b) reveals its vulnerability to ``critical mistakes''---frequently \emph{recommending} movies rated 1 or 2. Such failures can undermine user trust and harm companies' policy in keeping their customers. In contrast,  RAC (\(\alpha=0.05\)) cuts these critical errors by  75\%,
while incurring only a (15\%) reduction in average utility.

\section{Discussion and Future Work}
In this paper, we establish the decision-theoretic foundations of conformal prediction, demonstrating that valid prediction sets serve as a sufficient statistic for risk-averse agents seeking to optimize their value at risk. Building upon this framework, we developed an algorithmic interface to connect the predictions of any black-box predictive model to actions with marginal, distribution-free safety guarantees. 

\noindent While this paper primarily focuses on marginal safety guarantees, we recognize that in many practical scenarios, marginal guarantees alone may not suffice. Specifically, three types of conditional safety guarantees are often desirable: group-conditional safety (i.e., safety conditioned on certain characteristics of the covariate $x$), label-conditional safety (i.e., safety conditioned on the true label $y$), and action-conditional safety (i.e., safety conditioned on the action $a$ taken by the decision maker). Although we leave these aspects as avenues for future exploration, we believe that the majority of our findings can be systematically extended to these more complex cases.

\noindent Furthermore, while in many applications it is feasible to define a utility function that captures the preferences of the decision maker, in some contexts this may prove challenging. In such cases, we could  rely on  estimation of the utility  or  express preferences in a relative manner—indicating that one action is preferred over another in a given context. These considerations point to two promising directions for future work: first, examining the robustness of Risk-Averse Calibration under utility mis-specification, and second, developing algorithmic frameworks that require only access to preference functions, rather than explicit quantitative utility functions. Alternately, we could explore uncertainty quantification that is simultaneously useful for many downstream decision-makers with different utility functions. This approach does not require knowledge of the utility function of the specific downstream decision maker that we are interested in.  \citet{noarov2023high,roth2024forecasting} show that this is possible for expectation maximizing decision makers using refinements of (decision) calibration --- is the same possible for risk averse decision makers?

\section{Acknowledgments}
This work was supported  by the NSF Institute for CORE Emerging Methods in Data Science (EnCORE) and NSF grant FAI-2147212.
The authors wish to thank John Cherian, Natalie Collina, and Bruce D. Lee for helpful discussions. 




{\small
\setlength{\bibsep}{0.2pt plus 0.3ex}
\bibliographystyle{apalike}
\bibliography{references}
}

\newpage
\appendix
\section{Proofs}

\subsection{Proof of Proposition~\ref{maxmin_policy}} 
We prove that the risk-averse decision rule
\[
    a_{\rm RA}\bigl(C(x)\bigr) 
    \;:=\;
    \arg\max_{a \in \mathcal{A}}\; \min_{y \in C(x)} u(a, y)
\]
solves the minimax problem in \eqref{minimax}. 

\paragraph{Part 1: Upper bound for any arbitrary policy.}
Let $\pi(\cdot): 2^\mathcal{Y}\to\mathcal{A}$ be any policy, and let $C(\cdot)$ be a fixed set function satisfying 
\[
    \Pr_{(X, Y)\sim\mathcal{P}}\bigl[Y \in C(X)\bigr] \;\ge\; 1-\alpha.
\]
We construct a ``worst-case'' distribution in $\Omega$ for $\pi$. 

Pick any $x\in \mathcal{X}$ for which $C(x)\neq\emptyset$. Define a distribution $p^\ast(x,y)$ by
\[
    p^\ast(X = x) \;=\; 1, 
    \quad
    p^\ast(Y = y \mid X = x) \;=\; 
    \begin{cases}
      1 & \text{for some }y \in \arg\min_{z \in C(x)} u\bigl(\pi(C(x)), z\bigr),\\
      0 & \text{otherwise.}
    \end{cases}
\]
Under $p^\ast$, we have $Y \in C(X)$ almost surely (since $C(x)$ is nonempty and we place all mass on a label in $C(x)$). Hence $p^\ast \in \Omega$ because the marginal coverage constraint 
\[
    \Pr_{(X,Y)\sim p^\ast}[Y \in C(X)]
    =
    1 \;\ge\; 1-\alpha
\]
is satisfied. But under this distribution, the utility of $\pi(C(x))$ is forced to be
\[
    \min_{y \in C(x)} u\bigl(\pi(C(x)), y\bigr),
\]
since $Y$ is chosen (with probability 1) to be the worst-case label within $C(x)$. Thus, for this specific $x$, no matter how we choose $\pi$, its achievable value is at most 
$   \min_{y \in C(x)} u\bigl(\pi(C(x)), y\bigr)$. Also,
\[
    \min_{y \in C(x)} u\bigl(\pi(C(x)), y\bigr) \leq \max_{a \in \mathcal{A}} \min_{y \in C(x)} u\bigl(a, y\bigr),
\]

Because $x$ was arbitrary (among those with $C(x)\neq \emptyset$), repeating the same argument for each such $x$ yields 
\[
    \inf_{p \in \Omega} \nu^*(\pi, p) 
    \;\;\le\;\; 
    \inf_{x \;:\; C(x)\neq\emptyset}
    \max_{a \in \mathcal{A}} \min_{y \in C(x)} u\bigl(a, y\bigr).
\]
In other words, \emph{any} policy $\pi$ cannot achieve a value larger than the above infimum for the inner minimization in \eqref{minimax}.

\paragraph{Part 2: Achievability by the $\max\min$ policy.}
Next, we show that the policy
\[
    \pi^*\bigl(C(x)\bigr)
    \;=\;
    \arg\max_{a \in \mathcal{A}} \min_{y \in C(x)} u(a,y)
\]
matches the upper bound from Part~1 and is thus minimax optimal. Consider any $p \in \Omega$. 

Define
\[
    \nu(x) 
    \;:=\;
    \max_{a \in \mathcal{A}} 
    \min_{y \in C(x)}\,u(a,y).
\]
For those $x\in\mathcal{X}$ such that $C(x)$ is empty put $\nu(x) = \max_{a\in\mathcal{A}}\max_{y\in\mathcal{Y}} u(a, y)$.
We claim that with probability at least $1-\alpha$, the policy $a_{\rm RA}(C(x))$ achieves a utility at least $\nu(x)$. Indeed, on the event $\{Y\in C(X)\}$ (which has probability at least $1-\alpha$ by assumption), it holds that
\[
    u\Bigl(a_{\rm RA}\bigl(C(X)\bigr),\,Y\Bigr)
    \;\ge\;
    \min_{y\in C(X)} u\Bigl(a_{\rm RA}\bigl(C(X)\bigr), y\Bigr)
    \;=\;
    \nu(X).
\]
Thus, setting the target utility at each $x$ to $\nu(x)$ satisfies
\[
    \Pr_{(X,Y)\sim p}\Bigl[u\bigl(a_{\rm RA}(C(X)),Y\bigr)\;\ge\;\nu(X)\Bigr] 
    \;\ge\; 
    1-\alpha.
\]
By definition of $\nu^*(\cdot,\cdot)$, this implies
\[
    \nu^*\bigl(a_{\rm RA}, p\bigr) 
    \;\;\ge\;\; 
    \mathbb{E}_{X\sim p}\bigl[\nu(X)\bigr] 
    \;=\; 
    \mathbb{E}_{X\sim p} \Bigl[\max_{a}\;\min_{y\in C(X)}\,u(a,y) \bigr] \geq 
    \inf_{x \;:\; C(x)\neq\emptyset}
    \max_{a \in \mathcal{A}} \min_{y \in C(x)} u\bigl(a, y\bigr).
\]
Since $p \in \Omega$ was arbitrary, we have shown
\[
    \inf_{p \in \Omega} \nu^*\bigl(a_{\rm RA}, p\bigr)
    \;\;\ge\;\;
    \inf_{x \;:\; C(x)\neq\emptyset}
    \max_{a \in \mathcal{A}} \min_{y \in C(x)} u\bigl(a, y\bigr).
\]
Comparing with the upper bound in Part~1 establishes that $a_{\rm RA}$ attains the best possible (minimax) value. Hence 
\[
    \pi^*(x) 
    \;=\; 
    a_{\rm RA}\bigl(C(x)\bigr)
    \;=\;
    \arg\max_{a\in\mathcal{A}} \min_{y\in C(x)} u(a,y)
\]
solves the minimax problem \eqref{minimax}.

\subsection{Proof of Theorem \ref{prediction_set_equivalence}}

We give a constructive proof by showing how from each solution of RA-DPO we can construct a feasible solution of RA-CPO  without losing any utility, and vice versa. By applying this to the optimal solutions of both problems, we obtain the result of the theorem. 

\noindent\textbf{(I) From RA-DPO to RA-CPO.}
Suppose we have an feasible solution \(\bigl(a(\cdot), \tau(\cdot)\bigr)\) to the RA-DPO problem. 
Consider a pair $(a(\cdot), \nu(\cdot))$ such that $a: \mathcal{X} \to \mathcal{A}$ and $\nu: \mathcal{X} \to [0, u_{\rm max}]$. Here, we have $u_{\rm max} = max_a \max_y u(a,y)$, and as mentioned in Section~\ref{Sec:problem}, since $\nu$ is a utility certificate its value at any $x$ should be less than $u_{\rm max}$.  
Since $(a, \nu)$ is a feasible solution of RA-DPO, it satisfies the following: 
\[
\text{Pr}_{X,Y}\left[u(a(X), Y) \geq v(X) \right] \geq 1- \alpha. 
\]

Define a prediction set 
\begin{equation} \label{proof_equivqlence_C}
 C(x) 
\;=\; 
\Bigl\{\, 
y \:\mid\:
u\bigl(a(x),y\bigr) \;\ge\; \nu(x)
\Bigr\}.   
\end{equation}
In  words, \(C(x)\) is the set of labels \(y\) for which the utility 
\(u\bigl(a(x),y\bigr)\) is at least $\nu(x)$. By definition, we have
$$ \Pr \bigl[ Y \in C(X) \mid X= x \bigr] = \Pr\bigl[  u\bigl(a(X),Y\bigr) \geq \nu(X) \mid X = x \bigr]. $$
\noindent
As a result, we have
\begin{align*}
\Pr\bigl[\,Y \in C(X)\bigr] & = \mathbb{E}_X\left[ \Pr\bigl[ Y \in C(X) \mid X \bigr] \right] \\ 
& =   \mathbb{E}_X\left[ \Pr\bigl[ u(a(X), Y) \geq \nu(X) \mid X \bigr] \right] \\
& = \Pr\bigl[u(a(X), Y) \geq \nu(X)  \bigr] \\
& \geq 1 - \alpha.
\end{align*}
Hence, \(C(\cdot)\) satisfies the marginal coverage constraint of RA-CPO.

\noindent Next, we will improve the prediction sets $C$ to new prediction sets $\tilde{C}$ which satisfy the marginal guarantee but can potentially have larger value under the objective of RA-CPO. The basic idea is to consider points $x \in \mathcal{X}$ such that $C(x)$ is empty and augment an additional element to those empty sets. Recall that we defined $u_{\rm max} := \max_{a \in \mathcal{A}} \max_{y \in \mathcal{Y}} u(a,y)$. Hence, there exists at least one (action, label) pair, which we call $(a_{\rm max},y_{\rm max})$ such that $u_{\rm max} = u(a_{\rm max},y_{\rm max})$. Now, let us define
$$ \mathcal{X}_{\rm empty} = \{x \in \mathcal{X}: C(x) = \emptyset\}, $$
where $\emptyset$ denotes the empty set. We now update $C(\cdot)$ to $\tilde{C}(\cdot)$ as follows: 
\begin{align*}
& \text{- if } x \in \mathcal{X}_{\rm empty}: \quad \tilde{C}(x) = \{y_{\rm max}\},\\
& \text{- if } x \notin \mathcal{X}_{\rm empty}: \quad 
\tilde{C}(x) = C(x).
\end{align*}
Note that we have for any $x \in \mathcal{X}$ that $C(x) \subseteq \tilde{C}(x)$,. Hence,  $\tilde{C}(\cdot)$ satisfies the marginal coverage guarantee as $C(\cdot)$ is marginally valid.

\noindent
Next, we show that the RA-CPO objective under \(\tilde{C}(\cdot)\) 
is at least equal to the RA-DPO objective under \(\bigl(a(\cdot), \nu(\cdot)\bigr)\). 
Recall that the RA-CPO objective evaluated at $\tilde{C}(\cdot)$ is
\[
\E_X \Bigl[
    \max_{a\in\mathcal{A}}\;\min_{y\in \tilde{C}(X)}\,u(a,y)
\Bigr].
\]
To bound this objective value, we consider two cases based on whether of not $x$ belongs to $\mathcal{X}_{\rm empty}$. 

Consider first the case $ x \notin \mathcal{X}_{\rm empty}$. By definition of \(C(x)\) from \eqref{proof_equivqlence_C}, we have \(\min_{y \in C(x)} u(a(x),y) 
\;\ge\; \nu(x)\). Hence, for $x \notin \mathcal{X}_{\rm empty}$, by noting that $C(x) \neq \emptyset$, we have
\[
\max_{a\in \mathcal{A}}
\;\min_{y \in C(x)} u(a,y)
\;\;\ge\;\;
\min_{y \in C(x)} u\bigl(a(x),y\bigr)
\;\;\ge\;\;
\nu(x).
\]
Therefore, for $x \notin \mathcal{X}_{\rm empty}$, by noting that $\tilde{C}(x) = C(x)$, we have
\[
\max_{a\in \mathcal{A}}
\;\min_{y \in \tilde{C}(x)} u(a,y)
\;\;= \;\;
\max_{a\in \mathcal{A}}
\;\min_{y \in C(x)} u(a,y)
\;\;\ge\;\;
\nu(x).
\]
Now, let's consider the other case where $x \in \mathcal{X}_{\rm empty}$. For this case, we not that as $C(x) = \{y_{\rm max}\}$, and from the fact that for any $x \in \mathcal{X}$ we have $\nu(x) \leq u_{\rm \max}$, we can simply derive 
\[
\max_{a\in \mathcal{A}}
\;\min_{y \in \tilde{C}(x)} u(a,y)
\;\;= \;\;
u_{\rm max}
\;\;\ge\;\;
\nu(x).
\]

Therefore, putting the two cases above together, we have proven
\[
\E_X \biggl[
    \max_{a\in\mathcal{A}}\min_{y\in \tilde{C}(X)}u(a,y)
\biggr]
\;\;\ge\;\;
\E_X \bigl[\nu(X)\bigr]
\]
Hence, we have constructed a feasible solution to RA-CPO, namely $\tilde{C}(\cdot)$, that achieves an objective value for RA-CPO which is at least as big as the value of RA-DPO achieved by $(a(\cdot),\nu(\cdot))$. 
Thus, starting from an a solution of RA-DPO, we have constructed a 
solution to RA-CPO with at least the same objective value.

\noindent\textbf{(II) From RA-CPO to RA-DPO.}
Conversely, suppose we have a feasible solution \(C(\cdot)\) to RA-CPO, which is marginally valid, i.e.
\[
\Pr\bigl[Y \in C(X)\bigr] \;\ge\; 1-\alpha
\]
Define a the action policy $a(\cdot)$ and utility certificate $\nu(\cdot)$ as follows:
\[
a(x) 
\;:=\;
\arg\max_{a\in\mathcal{A}}\;
\min_{y\in C(x)} u(a,y),
\quad\quad\text{and}\quad\quad
\nu(x) = \max_{a\in\mathcal{A}}\;
\min_{y\in C(x)} u(a,y).
\]
It is now easy to see that 
$$\Pr\bigl[u(a(X), Y) \geq \nu(X)\bigr] = \Pr\bigl[Y \in C(X)\bigr] \geq 1-\alpha. $$

\noindent
Moreover, by definition of $\nu(x)$, we can easily deduce 
$$\mathbb{E}_X[\nu(X)] = \mathbb{E}_X[\max_{a\in\mathcal{A}}\;
\min_{y\in C(X)} u(a,y)].  $$

Thus, from a feasible solution of RA-CPO, we constructed a feasible solution to RA-DPO that attains the same 
objective value, proving the equivalence in the 
other direction.  

\subsection{Proof of Proposition~\ref{optimal_prediction_set}}
\begin{proof}[Proof of Proposition~\ref{optimal_prediction_set}]
Fix any instance $x\in\mathcal{X}$ and a coverage value $t\in[0,1]$. Recall from \eqref{theta} that
\[
    \btheta(x,t)
    \;=\;
    \max_{a\in\mathcal{A}}
    \text{quantile}_{1-t}\bigl[u(a,Y)\mid X=x\bigr],
    \quad
    \ba(x,t)
    \;=\;
    \arg\max_{a\in\mathcal{A}}
    \text{quantile}_{1-t}\bigl[u(a,Y)\mid X=x\bigr].
\]
We want to show that among all sets $C$ with $\Pr[Y\in C \mid X=x]\ge t,$ the set
\[
    C(x,t)
    \;=\;
    \bigl\{\,y \in \mathcal{Y} \,:\, u\bigl(\ba(x,t),y\bigr)\;\ge\;\btheta(x,t)\bigr\}
\]
maximizes the risk-averse utility $\nu_{\rm RA}(C)=\max_{a\in\mathcal{A}}\min_{y\in C}\,u(a,y),$ and the maximum value is $\btheta(x,t)$.

\paragraph{Step 1: Any set $C$ with coverage $\ge t$ has risk-averse utility at most $\btheta(x,t)$.}
Take an arbitrary set $C\subseteq\mathcal{Y}$ satisfying
\[
    \Pr\bigl[Y\in C \mid X=x\bigr] \;\ge\; t.
\]
Then for any action $a\in\mathcal{A}$,
\[
    \min_{y\in C} u(a,y)
    \;\;\le\;\; \text{quantile}_{1-t}\bigl[u(a,Y)\mid X=x\bigr].
\]
(The reason is that with probability at least $t$, $Y$ lies in $C$, and so the $(1-t)$-quantile of $u(a,Y)$ cannot be smaller than the smallest utility on this event.) Taking the maximum over $a$ yields
\[
    \max_{a\in\mathcal{A}} \min_{y\in C} u(a,y)
    \;\;\le\;\;
    \max_{a\in\mathcal{A}} \text{quantile}_{1-t}\bigl[u(a,Y)\mid X=x\bigr]
    \;=\;
    \btheta(x,t).
\]
Hence no set with coverage at least $t$ can achieve risk-averse utility larger than $\btheta(x,t)$.

\paragraph{Step 2: The set $C(x,t)$ attains coverage $t$ and achieves $\btheta(x,t)$.}
Consider $C(x,t) = \{\,y: u(\ba(x,t),y)\ge \btheta(x,t)\}$. By definition of the $(1-t)$-quantile, we have
\[
    \Pr\bigl[u(\ba(x,t),Y) \;\ge\; \btheta(x,t)\mid X=x\bigr]
    \;\ge\; t,
\]
which implies $\Pr[Y \in C(x,t)\mid X=x]\ge t$. 
Moreover, for every $y\in C(x,t)$, by construction 
\[
    u\bigl(\ba(x,t),y\bigr) \;\ge\; \btheta(x,t),
\]
so
\[
    \min_{y\in C(x,t)} u\bigl(\ba(x,t),y\bigr) 
    \;\ge\;
    \btheta(x,t).
\]
Thus
\[
    \nu_{\rm RA}(C(x,t)) 
    \;=\; 
    \max_{a\in\mathcal{A}} \min_{y\in C(x,t)} u(a,y)
    \;\ge\; 
    \min_{y\in C(x,t)} u\bigl(\ba(x,t),y\bigr)
    \;\ge\;
    \btheta(x,t).
\]
Combining both steps shows that $C(x,t)$ is an optimal choice among all sets with coverage at least $t$, and its risk-averse utility equals $\btheta(x,t)$. 
\end{proof}

\subsection{Proof of Theorem \ref{strong_duality}} \label{app:proof_duality}
We start from the reparametrization of RA-CPO given in 
\eqref{opt_reparametrization}: 

\begin{equation} \label{opt_reparametrization2}
\tag{Reparametrization of RA-CPO}
\begin{aligned}
& \underset{t: \mathcal{X} \to [0,1]}{\text{maximize}} & & \E_X \bigl[\btheta(X, t(X))\bigr] \\
& \text{subject to:} & & \E_X\bigl[t(X)\bigr] \,\ge\, 1-\alpha.
\end{aligned}
\end{equation}

We will further reparametrize this optimization problem and find equivalent relaxations. To do so, let us define 
\begin{equation} \label{rho-t}
\rho(x, t) = \1[t \leq t(x)].
\end{equation}
Also, we will need to consider the derivative of the function $\btheta(x,t)$ in terms of its second argument $t$. Since the function $\btheta$ can be discontinuous, we will have to consider its generalized derivative (i.e. consider delta functions). More precisely, let $\btheta^{'}(x, .):\mathbb{R} \rightarrow\mathbb{R^*}$ where $\mathbb{R^*}$ is the space of functionals on $\mathbb{R}$, such that $\btheta^{'}(x, .)$ is the generalized derivative of $\btheta(x, .)$. In other words, for any real values $a$ and $b$,
\begin{align*}
    \int_a^b \btheta^{'}(x, t) dt = \btheta(x, b) - \btheta(x, a).
\end{align*}
We can just think of $\btheta^{'}(x, t)$ as the derivative $\frac{d}{d t} \btheta(x, t)$.  We can then rewrite the objective of our optimization problem as
\begin{align*}
    \E_X \left[\btheta(X, t)\right] = u_{\rm \max} + \E_X \int_{t=0}^1 \rho(X,t) \btheta^{'}(X, t) dt,
\end{align*}
where we used the fact that $\btheta(x, 0) = u_{\rm max}$ for any $x \in \mathcal{X}$ (by definition), and $\btheta(x,t) - \btheta(x,0) = \int_{0}^t \btheta'(x,t) dt$. Similarly, we can rewrite the constraint as,
\begin{align*}
    \E_X\left[t(X)\right] = \E_X\int_{t=0}^1 \rho(X, t)dt.
\end{align*}
Given the above notation and relations, we can write down the following equivalent reparametrization of \eqref{opt_reparametrization2}. The optimization variable here is the function $\rho(x,t)$ which is a step function according to \eqref{rho-t}. We further note that any such step function defined on the unit interval can be equivalently thought of as a non-increasing function on the unit interval which only takes its value in the set $\{0,1\}$. Hence we arrive at the following integer program that is an equivalent reparametrization of \eqref{opt_reparametrization2} as well as the RA-CPO:
\begin{equation}
\label{inetegr_program}
\tag{Integer Program}
\begin{aligned}
& \underset{\underset{\forall x \in \mathcal{X},t\in [0,1]}{\rho(x,t) \in \{0,1\}}}{\text{maximize}} & & \int_{ \mathcal{X}} \int_{t=0}^1 \rho(x,t) p(x) \btheta^{'}(x, t) dxdt\\
& \text{subject to:} & & \int_{\mathcal{X}}\int_{t=0}^1 p(x) \rho(x, t)dx dt \geq 1-\alpha \\
& & & \rho(x,t) = \text{non-increasing in } t
\end{aligned}
\end{equation}
We now consider a relaxation of the above integer program to the following convex program. As we will see later, this relaxation becomes equivalent to the above integer program as every solution of the relaxed program would correspond to a solution of the integer program. However, for now, let us focus on the following continuous relaxation whose variable $\rho(x,t)$ can take values in the interval $[0,1]$ (in contrast to the original integer program in which $\rho$ could take its value only in the set $\{0,1\}$): 
\begin{equation} \label{relaxed_program}
\tag{Relaxed Program}
\begin{aligned}
& {\text{maximize}} & & \int_{ \mathcal{X}} \int_{t=0}^1 \rho(x,t) p(x) \btheta^{'}(x, t) dxdt\\
& \text{subject to:} & & \int_{\mathcal{X}}\int_{t=0}^1 p(x) \rho(x, t)dx dt \geq 1 - \alpha \\
& & & \rho(x,t) \in [0,1] \quad \forall x \in \mathcal{X},t\in [0,1] \\
& & & \rho(x,t) = \text{non-increasing in } t
\end{aligned}
\end{equation}
Here, the ``optimization variable'' $\rho(x,t)$ belongs to an infinite-dimensional space. Hence, in order to be fully rigorous, we will need to use the duality theory developed for general linear spaces that are not necessarily finite-dimensional. For a reader who is less familiar with infinite-dimensional spaces, what appears below is a direct extension of the duality theory (i.e. writing the Lagrangian) for convex programs in finite-dimensional spaces. 

\noindent Let $\mathcal{F}$ be the set of all measurable function defined on $\mathcal{X} \times [0,1]$. Note that $\mathcal{F}$ is a linear space.  Let $\Omega$ be the set of all the measurable functions on $\mathcal{X} \times [0,1]$ which are non-increasing in $t$ and are bounded between $0$ and $1$; I.e. 
\begin{equation}
\Omega = \left\{ \rho \in \mathcal{F} \text{ s.t. }  \rho: \mathcal{X} \times [0,1] \to [0,1]; \forall x \in \mathcal{X}:   \rho(x,t) \text{ is non-increasing in } t   \right\}
\end{equation}
Note that $\Omega$ is a convex set. We can then rewrite the \eqref{relaxed_program} as follows: 
\[
\begin{aligned}
& {\text{maximize}} & & \int_{ \mathcal{X}}\int_{t=0}^1 \rho(x,t) p(x) \btheta'(x,t) dxdt\\
& \text{subject to:} & & \int_{\mathcal{X}}\int_{t=0}^1 p(x) \rho(x, t)dx dt - (1 - \alpha) \geq 0\\
& & & \rho \in \Omega
\end{aligned}
\]
Moreover, let us define the functional $F: \mathcal{F} \to \mathbb{R}$ as 
\begin{equation} \label{F_def}
F(\rho) = \int_{ \mathcal{X} \times [0,1] }\rho(x,t) p(x) \btheta'(x,t) dxdy, 
\end{equation}
and also define the functional $G: \mathcal{F} \to \mathbb{R}$ as 
\begin{equation} \label{G}
G(\rho) = \int_{\mathcal{X} \times [0,1]} \rho(x, t) p(x) dx dt - (1-\alpha). 
\end{equation}

\noindent Using the above-defined notation, our  program becomes: 
\[
\begin{aligned} \label{primary_proper}
& {\text{maximize}} & & F(\rho)\\
& \text{subject to:} & & G(\rho) \geq 0 \\
& & & \rho \in \Omega
\end{aligned}
\]

Note that the feasibility set of the above program is non-empty, as  $\rho(x,t) = 1 - \alpha $, for all $(x,t) \in \mathcal{X} \times [0,1]$, is a feasible point. Also, $F$ and $G$ are linear functionals over $\rho$. We can now use the duality theory of convex programs in vector spaces (See Theorem 1, Section 8.3 of \cite{luenberger1997optimization}. 
Specifically, let $\text{OPT}$ be the optimal value achievable in the above  program. Then,  there exists a scalar $\beta \geq 0$ such that the following holds:  
\begin{equation} \label{dual_gen}
\text{OPT} = \sup_{\rho \in \Omega} \left\{ F(\rho) +  \beta G(\rho)   \right\},
\end{equation}
Here, note that $\beta$ is the usual Lagrange multiplier.

\noindent By using \eqref{F_def}, in order to solve the optimization in \eqref{dual_gen} we need to solve the following optimization:
\begin{align*} 
& \sup_{\rho \in \Omega} \left\{
\int_{\mathcal{X} \times [0,1]} p(x)\rho(x, t) \left(  \btheta'(x,t) + \beta  \right)  dx dt
\right\} - \beta (1-\alpha).
\end{align*}
We denote the optimal solution of the above optimization problem by $\rho_\beta^*(x,t)$. From the above optimization problem, it is clear that the optimal solution can be determined individually for every $x \in \mathcal{X}$. We will use Lemma~\ref{joz_be_joz}, provided below, to characterize the optimizer of the above optimization. From the lemma, and assuming that, for $\beta >0$,  the maximizer of $\btheta(x,t) + \beta t$ is unique over $t$, almost surely for every $x \in \mathcal{X}$, we obtain:
\begin{equation} \label{rho_beta}
\rho^*_\beta(x,t) = \mathbf{1}\{t \leq t^*(x)\},
\end{equation}
where 
$$ t^*(x) = \arg\max_{s \in [0,1]} \int_{t=0}^s  \left(\btheta(x,t) +  \beta  \right) dt = \arg\max_{s \in [0,1]} \left\{\btheta(x,s) dt +  \beta s  \right\} := \bg(x,\beta).$$

\noindent And the value of $\beta$ should then be chosen such that this optimal solution satisfies the coverage constraint.   

\noindent  We finally note that the optimal solution $\rho_{\beta}(x,t)$ given in \eqref{rho_beta} is integer valued. As a result, there is a zero relaxation gap from the \eqref{inetegr_program}  to the \eqref{relaxed_program}.

\begin{lemma} \label{joz_be_joz}
Let $\theta: [0,1] \to \mathbb{R}$. Also, let $\Omega$ be the set of all the integrable functions $\rho: [0,1] \to [0,1] $ which are non-decreasing. Consider the following optimization problem: 
\[
\begin{aligned} 
& \max_{\rho \in \Omega}   \int_{0}^1 \theta'(t) \rho(t) dt, 
\end{aligned}
\]
where $\theta'$ denotes the (generalized) derivative of $\theta$ with respect to $t$ -- i.e., $\theta(a) - \theta(b) = \int_{a}^b \theta'(t)dt$. Then, the set of solutions of the above optimization problem consists of functions $\rho^*$ such that
$$ \rho^*(t) \in {\rm{ConvexHull}} \left( \left\{  \mathbf{1}[t \leq t^*]\,; \quad t^* \in \arg\max_{t \in[0,1]} \theta(t) \right\} \right).$$
As a corollary, if $\theta$ has a unique maximizer $t^*$, then its corresponding $\rho^*(t) = \mathbf{1}[t \leq t^*]$ is the unique solution of the above optimization problem.
\end{lemma}
\begin{proof}
For every $\rho \in \Omega$ write using integration by parts: 
$$ \int_{0}^1 \rho(t) \theta'(t) dt = \rho(1) \theta(1) - \rho(0) \theta(0) - \int_0^1 \rho'(t) \theta(t) dt   $$
Let us define $\theta_{\rm max} := \max_{t\in [0,1]} \theta(t)$. Since $\rho(t)$ is a non-increasing function, we have 
$$ - \int_0^1 \rho'(t) \theta(t) dt \stackrel{(a)}{\leq} -\theta_{\rm max} \int_{0}^1 \rho'(t) dt = \theta_{\rm max} (\rho(0) - \rho(1)),$$
where the step (a) is obtained because $-\rho(t)$ is non-negative.
As a result, we obtain 
\begin{align*}
\int_{0}^1 \rho(t) \theta'(t) dt 
& \leq \theta_{\rm max} \left(\rho(0) - \rho(1) \right) + \rho(1) \theta(1) - \rho(0) \theta(0) \\
&=
\left( \theta_{\rm max} - \theta(0) \right) \left(\rho(0) - \rho(1) \right) + \rho(1) \left( \theta(1) - \theta(0) \right)  \\
&\stackrel{(b)}{\leq} 
\theta_{\rm max} - \theta(0), \\
\end{align*}
where step (b) is obtained since $\rho(0), \rho(1) \in [0,1]$ and $\theta(1) \leq \theta_{\rm max}$. 

\noindent  Now, it it easy to see that if $t^*$ is such that $\theta(t^*) = \theta_{\rm max}$ then both steps (a), (b) will be equality (instead of an inequality) for the following function
$$\rho^*(t) =  \mathbf{1}[t \leq t^*].$$
On the other hand, for step (a) to be tight we must have the following: For every point $t$ such that $\rho'(t) < 0$, we have $\theta(t) = \theta_{\rm max}$. This shows that an optimal solution must be in the convex hull defined in the theorem, and 
hence, the result of the theorem follows. The uniqueness also follows similarly.
\end{proof}

\noindent\textbf{Breaking Ties.} 
Theorem~\ref{strong_duality} was stated and proved under the assumption that the function $\bg$, given in \eqref{g-beta}, is well-defined -- i.e. for any $\beta>0$ with probability $1$ we have that the maximizer of the function $\btheta(x,s) + \beta s$, over the choice of $s \in [0,1]$, is unique. We now focus on the case of having ties, i.e., special, pathological settings where more than one maximizer exists with positive probability. We will explain an approach on how ties between the maxima could be broken in such cases in a lossless manner. The idea is simple -- we will perturb the distribution $p(y|x)$ by a small noise to break the ties. 

\noindent Formally, fix an arbitrarily small value $\epsilon > 0$. Consider a point $x$, and a value $\beta > 0$ such that the maximizer in not unique. Sample a probability-vector $\pi_{x} \in \Delta_{|\mathcal{Y}|}$ uniformly at random (i.e. $\pi$ is a random distribution over $\mathcal{Y}$). Then, let $p_{\epsilon}(y|x) = (1-\epsilon)p(y|x) + \epsilon \pi_x.$
We then construct the functions $\btheta(x,t)$ and  $\bg(x,\beta)$ exactly as before but with the perturbed distribution $p_\epsilon$. We call this setting the perturbed setting.

\noindent We denote the perturbed distribution by $p_\epsilon(x,y) = p(x) p_\epsilon(y|x)$. Note that the original distribution of the data has the form $p(x,y) = p(x) p(y|x)$, and for every $x \in \mathcal{X}$, we have $|| p(y|x) - p_\epsilon(y|x)||_{\rm TV} \leq \epsilon$, where TV denotes the total variation distance. 

We will show that:
\begin{enumerate}
\item In the perturbed setting, for any $\beta >0$ the maximizer of $\btheta(x,s) + \beta s$ over the choice of $s\in[0,1]$ is unique with probability $1$.   
\item Consider a prediction set constructed under the perturbed distribution with marginal coverage at least $1-\alpha$ and average maxmin utility (under the objective of RA-CPO) denoted by $\bar{u}$. Then, this prediction set achieves marginal coverage of at least $1-\alpha - \epsilon$ and utility $\bar{u}$ under the original distribution. The other direction is also true: A marginally valid prediction set of the original distribution with maxmin utility $\bar{u}$, achieves $1-\alpha - \epsilon$ coverage under the perturbed distribution with the same expected utility. Also, note that any prediction set with marginal coverage $1-\alpha - \epsilon$ and  expected maxmin utility $\bar{u}$ can be easily made into a prediction set of marginal coverage $1- \alpha$ with expected maxmin utility $\bar{u} - c \epsilon $, where $c$ is a finite number independent of $\epsilon$ 
(simply, take a small part of the space $\mathcal{X}$ and assign the full set $\mathcal{Y}$ as the prediction set to the points in that part).

\item From point (1) above, we can use the result of the theorem and construct the optimal prediction sets for the perturbed distribution using the exact same procedure explained in the theorem. 
\item From  point (2), and by choosing an arbitrarily small value $\epsilon$, we can obtain a solution of the original problem with small, $c \epsilon$, loss in the optimal utility. The gap will vanish by decreasing $\epsilon$.
\end{enumerate}
The above sequence of points shows that the randomized tie-breaking explained above is valid. Hence, we need to argue the validity  of points (1) and (2). 

\noindent  We first prove point (2) since it is easier. Consider  prediction sets $C(x)$, $x \in \mathcal{X}$, that achieve coverage at least $1-\alpha$ under the perturbed distribution. Note that for every $x \in \mathcal{X}$ the perturbed distribution $p_\epsilon(x,y)$ can be written as $p_\epsilon(x,y) = p(x) p_{\epsilon}(y|x)$. Further note that for a given $x$ we have $|| p(y|x) - p_\epsilon(y|x)||_{\rm TV} \leq \epsilon$, where TV is the total variation distance. Hence, we can conclude that (i) the marginal coverage of $C(x)$ under the original distribution $p(x,y)$ is at least $1-\alpha - \epsilon$, and (ii) since both distributions induce the same marginal distribution on $x$ then the expected maxmin utility (objective of RA-CPO) is the same under both distributions. The same argument can be made for a prediction set which is marginally valid under the original distribution. Finally, the last part o point (2) can be proven easily. Consider prediction sets $C(x)$ with marginal coverage $1-\alpha - \epsilon$. Let $A \in \mathcal{X}$
such that for every $x \in A$ the conditional coverage of $C(x)$, conditioned on $X = x$, is at most $1-\alpha/2$. Then, it is not hard to show using the Markov inequality that $p(A) \geq \alpha / (2(1-\alpha))$. Now, among the set $A$ we chose a subset $A'$ such that $p(A') = 2\epsilon / \alpha$. Note that for $\epsilon$ sufficiently small this is always possible since $p(x)$ is an atom-less (continuous) distribution. Finally, for every $x \in A'$ we will change its prediction set $C(x)$ to the full set $\mathcal{Y}$. It is now easy to see that the new prediction sets will have marginal coverage at least $1-\alpha$. Also, since we are changing only the prediction sets corresponding to points inside the set $A'$, then the utility loss will be at most $p(A') \times u_{\rm max}$, where $u_{\rm max} = \max_a \max_y u(a,y)$. Hence, the utility loss will be at most $c\epsilon$ with $c = 2u_{\rm max}/\alpha $.     

\noindent  Let us now proceed with arguing the validity of point (1). Fix a point $x \in \mathcal{X}$.  Then from \eqref{theta} it is clear that $\btheta(x,s)$ is piece-wise constant function over $s$ with at most $|\mathcal{Y}|+1$ jumps, where we are always considering $s=1$ as a jump.  Let the set of jumps of the function $\btheta(x,s)$ be denoted  by $s_1, s_2, \cdots, s_J$, where $J \leq |\mathcal{Y}+1|$.

\noindent  Since $\btheta(x,s)$ is piece-wise constant, then for $\beta >0$ maximizers of $\btheta(x,s) + \beta s$, over $s \in [0,1]$, happen only at the points where $\btheta(x,s)$ has a jump. Now consider two different jumps, $s_i, s_j$, of the function $\btheta(x,s)$. If both of these jumps are maximizers for the function $\btheta(x,s) + \beta s$, then we must have: 
\begin{equation} \label{beta-jump}
\btheta(x,s_i) + \beta s_i = \btheta(x,s_j) + \beta s_j \longrightarrow  \beta = \left(\btheta(x,s_j) - \btheta(x,s_i) \right)/\left(s_i - s_j \right). 
\end{equation}
This shows that for every pair of jumps $(s_i,s_j)$ there exists at most one $\beta$ for which both $s_i, s_j$ are the maximizers of the function $\btheta(x,s) + \beta s$. Now, since we have at most $(|\mathcal{Y}| +1)^2$ of such pairs of jumps $(s_i,s_j)$, then we can conclude the following: For every $x \in \mathcal{X}$, the set of $\beta$'s such that the maximizer of the function $\btheta(x,s) + \beta s$ over $s$ is not unique is a finite set with at most $(|\mathcal{Y}| + 1)^2$ elements. Now, for a given $x$, take one of such $\beta$'s and denote it by $\beta_x$. We claim that the following set,
$$\mathcal{Z} = \{ x' \in \mathcal{X} \text{ s.t. the maximizer of $\btheta(x',s) + \beta_x s$ is not unique} \}, $$
has measure $0$ -- i.e. $p(\mathcal{Z}) = 0$. The reason is that for any $x' \in \mathcal{X}$, the set of jumps for the function $\btheta(x',s)$, which we denote by $\{s'_1, s'2, \cdots, s'_K\}$ is a function of the random noise in the conditional probability $p_\epsilon(y|x')$. As a result, the probability that a pair of jumps $s'_i, s'_j$ satisfy the following equation (which is derived similar to \eqref{beta-jump}) 
$$ \beta_x = \frac{\btheta(x', s'_j) - \btheta(x', s'_i) }{s'_i - s'_j} $$
is zero. This is because for $x' \neq x$ the noise added inside $p_\epsilon(y|x')$ is chosen completely independent of $x$ (and hence it's independent of the value $\beta_x$). From this fact, and the fact that each $x$ has only a finite number of $\beta_x$'s we can conclude that the for any $x$ the following set
\begin{align*}
\mathcal{Z}_x = \{ x' \in \mathcal{X} \text{ s.t. there } & \text{exists a $\beta >0$ for which} \\ &\text{both functions $\btheta(x',s) + \beta s$  and $\btheta(x,s) + \beta s$ have at least two maximizers} \}, 
\end{align*}
has measure $0$. 
Using this fact, and the fact that $p(x)$ is atom-less we can now conclude the proof of point (1).  The reason is that, if (1) is not true, then there exists a value $\beta >0$ such that with non-zero probability over $x$ the maximizer of $\btheta(x,s) + \beta s$ is non-unique. Hence, there exists a point $x$ for which the set $\mathcal{Z}_x$ defined above must have non-zero measure. And this is a contradiction.

\subsection{Proof of Theorem~\ref{coverage_theorem}}
We have:
\begin{align}
    \Pr\bigl[Y_{\rm test} \in C_{\rm RAC}(X_{\rm test})\bigr]
    &\overset{\mathrm{(a)}}{=}\Pr\bigl[Y_{\rm test} \in \hat{C}\bigl(X_{\rm test}; \hat{\beta}_{Y_{\rm test}}\bigr)\bigr]
    \nonumber\\
    & = \E [\1[Y_{\rm test} \in \hat{C}\bigl(X_{\rm test}; \hat{\beta}_{Y_{\rm test}}\bigr)]]
    \nonumber\\
    &\overset{\mathrm{(b)}}{=}
    \E \bigg[\frac{1}{n+1}\bigg(\sum_{i=1}^n \1[Y_i \in \hat{C}\bigl(X_i; \hat{\beta}_{Y_{\rm test}}\bigr)] + \1[Y_{\rm test} \in \hat{C}\bigl(X_{\rm test}; \hat{\beta}_{Y_{\rm test}}\bigr)]\bigg)\bigg]\nonumber\\
    &\overset{\mathrm{(c)}}{\geq}
    1-\alpha.
\end{align}
where, (a) comes form the definition of the prediction set. (b) comes from the fact that $$\Bigl\{\,
      \bigl(X_1,Y_1,\hat{\beta}_{Y{\rm test}}\bigr),\;\dots,\;
      \bigl(X_n,Y_n,\hat{\beta}_{Y{\rm test}}\bigr),\;
      \bigl(X_{\rm test},Y_{\rm test},\hat{\beta}_{Y_{\rm test}}\bigr)
    \Bigr\}$$
    are exchangeable, which is due to the fact that (i) the exchangeability of the original $(n+1)$ pairs \(\{(X_i,Y_i)\}\cup\{(X_{\rm test},Y_{\rm test})\}\), and (ii) the symmetric way in which Algorithm~\ref{alg} assigns $\hat{\beta}_y$ to each $y\in\mathcal{Y}$. Finally, (c) follows from the definition of $\hat{\beta}_{Y_{\rm test}}$.

\section{Utility function for medical experiment}\label{utility_AI}
\noindent Our results and findings in the medical experiment of section \ref{exp:med}, can be reproduced with any other reasonable design of utility function. The goal of that experiment is not to capture a precise characterization of difficulties and consequences in medical decision making but rather to pinpoint the advantages of a risk averse calibration approach in sensitive tasks like medical decision making. Of course, in real world scenarios, a more comprehensive approach is needed to define a principled utility function that captures the interests of all the involving parties. That being said, for the sake of proof of concept, we designed a utility matrix using the ChatGPT o1 model by OpenAI. The following is an AI generated text justifying the proposed utility matrix.

\noindent\textbf{Clinical Justification of the Utility Matrix}

\noindent The utility matrix presented in Table~\ref{tab:utility_1} reflects the balance of benefits and harms associated with different medical actions for each true clinical condition. Each utility value is determined based on standard clinical guidelines and evidence-based practices, ensuring that the chosen actions optimize patient outcomes while minimizing potential risks.

\noindent \textbf{Normal (No Disease)}
\begin{itemize}
    \item \textbf{No Action = 10} \\
    For a patient who is truly healthy, no intervention is optimal as it avoids unnecessary costs, side effects, and patient anxiety. Unwarranted use of antibiotics or quarantine measures can lead to adverse effects and resource wastage \cite{ref1}.
    
    \item \textbf{Antibiotics = 2} \\
    Prescribing antibiotics to a healthy individual can contribute to antimicrobial resistance and cause side effects without any clinical benefit \cite{ref2}.
    
    \item \textbf{Quarantine = 2} \\
    Quarantining a healthy person imposes unnecessary social and psychological burdens without providing any medical advantage \cite{ref3}.
    
    \item \textbf{Testing = 4} \\
    While testing can confirm the absence of disease, routine testing in healthy individuals is often not cost-effective and may lead to unnecessary follow-up procedures \cite{ref4}.
\end{itemize}

\textbf{Pneumonia}
\begin{itemize}
    \item \textbf{No Action = 0} \\
    Untreated pneumonia can lead to rapid deterioration and increased mortality, making inaction highly detrimental \cite{ref5}.
    
    \item \textbf{Antibiotics = 10} \\
    Timely administration of appropriate antibiotics is crucial for treating bacterial pneumonia, improving survival rates and reducing complications \cite{ref5, ref6}.
    
    \item \textbf{Quarantine = 3} \\
    While some forms of pneumonia may be contagious, standard infection control measures are generally more beneficial than full quarantine, especially when bacterial pneumonia is suspected \cite{ref5}.
    
    \item \textbf{Testing = 7} \\
    Diagnostic tests such as chest imaging and sputum cultures are essential for confirming pneumonia and guiding antibiotic therapy \cite{ref5}.
\end{itemize}

\textbf{COVID-19}
\begin{itemize}
    \item \textbf{No Action = 0} \\
    Ignoring a COVID-19 infection can result in severe disease progression and widespread transmission, making inaction extremely harmful \cite{ref3}.
    
    \item \textbf{Antibiotics = 3} \\
    Since COVID-19 is viral, antibiotics are generally only useful if there is a suspected secondary bacterial infection \cite{ref7}.
    
    \item \textbf{Quarantine = 10} \\
    Quarantining individuals with COVID-19 is essential for controlling the spread of the virus and protecting public health \cite{ref3}.
    
    \item \textbf{Testing = 8} \\
    Confirmatory testing is vital for diagnosing COVID-19 and guiding appropriate interventions, including quarantine and specific therapies \cite{ref3}.
\end{itemize}

\textbf{Lung Opacity}
\begin{itemize}
    \item \textbf{No Action = 1} \\
    Ignoring lung opacities can lead to missed diagnoses of serious conditions such as malignancies or tuberculosis, posing significant risks \cite{ref8}.
    
    \item \textbf{Antibiotics = 4} \\
    Empirical antibiotic therapy may be beneficial if an infectious etiology is suspected, but it is not universally appropriate and may lead to resistance \cite{ref2, ref5}.
    
    \item \textbf{Quarantine = 4} \\
    Quarantine may be necessary if the underlying cause of the opacity is contagious, but many causes do not require isolation \cite{ref2}.
    
    \item \textbf{Testing = 10} \\
    Comprehensive diagnostic evaluation is crucial for determining the exact cause of lung opacities, guiding targeted treatment and preventing misdiagnosis \cite{ref8}.
\end{itemize}
\begin{table}[t]
\centering
\small  
\setlength{\tabcolsep}{2pt}  
\renewcommand{\arraystretch}{0.9}  
\begin{tabular}{lcccc}
\toprule
\textbf{True Label} & \textbf{No Action} & \textbf{Antibiotics} & \textbf{Quarantine} & \textbf{Testing} \\
\midrule
Normal       & 10  & 2  & 2  & 4  \\
Pneumonia    & 0   & 10 & 3  & 7  \\
COVID-19     & 0   & 3  & 10 & 8  \\
Lung Opacity & 1   & 4  & 4  & 10 \\
\bottomrule
\end{tabular}
\caption{Utility matrix for the four-class chest X-ray task.}
\label{tab:utility_1}
\end{table}

\textbf{Key Takeaways}
\begin{enumerate}
    \item \textbf{Benefit vs. Harm}: The utility scores balance the potential benefits of medical interventions against their associated risks and costs.
    \item \textbf{Disease-Specific Standard of Care}: Treatments are aligned with established clinical guidelines specific to each condition.
    \item \textbf{Avoidance of Unnecessary Interventions}: The matrix discourages overtreatment in healthy individuals to prevent adverse effects and resource wastage.
\end{enumerate}

\noindent Overall, the utility matrix aligns with standard clinical guidelines by advocating for appropriate treatment of infections, isolation of contagious diseases, thorough diagnostic evaluations for ambiguous findings, and avoiding unnecessary interventions in healthy patients.

\end{document}

%% file: notation.tex
\newcommand{\R}{\mathbb{R}}

\newcommand{\E}{\mathbb{E}}

\newcommand{\cX}{\mathcal{X}}
\newcommand{\cD}{\mathcal{D}}
\newcommand{\cA}{\mathcal{A}}
\newcommand{\cY}{\mathcal{Y}}

\def\R{\mathbb{R}}

\def\ba{{\boldsymbol a}}

\def\btheta{{\boldsymbol \theta}}
\def\bg{{\boldsymbol g}}

\def\hbtheta{{\boldsymbol {\hat{\theta}}}}

\def\0{{\boldsymbol 0}}
\def\1{{\boldsymbol 1}}

\def\1{{\boldsymbol 1}}

\newtheorem{theorem}{Theorem}[section]

\newtheorem{lemma}[theorem]{Lemma}
\newtheorem{corollary}[theorem]{Corollary}
\newtheorem{propo}[theorem]{Proposition}

\newtheorem{remark}[theorem]{Remark}

\definecolor{gold}{RGB}{212,175,55}